\documentclass[11pt]{article}

\usepackage[colorlinks=true,linkcolor=blue,urlcolor=black,citecolor=blue,breaklinks]{hyperref}
\usepackage{color}
\usepackage{enumerate}
\usepackage{graphicx}
\usepackage{algorithm, algcompatible}
\usepackage{varwidth}
\usepackage{mathrsfs}
\usepackage{mathtools}
\usepackage[labelfont=bf,font={footnotesize}]{caption}
\usepackage[labelfont=bf]{subcaption}
\usepackage{wrapfig}
\usepackage{overpic}
\usepackage{multirow}

\usepackage[sort&compress,numbers]{natbib}
\usepackage{fullpage}
\usepackage{authblk}

\usepackage{amsmath,amsthm,amssymb,colonequals,etoolbox}

\newtoggle{draft}
\toggletrue{draft}

\newtheorem{thm}{Theorem}[section]

\newtheorem{lem}[thm]{Lemma}
\newtheorem{prop}[thm]{Proposition}

\newtheorem{assume}{Assumption}[section]

\numberwithin{equation}{section}

\newcommand{\calV}{\mathcal{V}}

\newcommand{\R}{\ensuremath{\mathbb{R}}}

\newcommand{\norm}[1]{\lVert #1 \rVert}
\newcommand{\bignorm}[1]{\left\lVert #1 \right\rVert}

\newcommand{\bigip}[2]{\ensuremath{\left\langle #1, #2 \right\rangle}}
\newcommand{\ip}[2]{\ensuremath{\langle #1, #2 \rangle}}

\newcommand{\E}{\mathbb{E}}
\newcommand{\abs}[1]{\ensuremath{| #1 |}}
\newcommand{\bigabs}[1]{\ensuremath{\left| #1 \right|}}

\newcommand{\bigceil}[1]{\left\lceil #1 \right\rceil}

\newcommand{\ind}{\mathbf{1}}

\newcommand{\leb}{\mu_{\mathsf{Leb}}}
\renewcommand{\Pr}{\mathbb{P}}
\newcommand{\T}{\mathsf{T}}
\newcommand{\TM}{\mathcal{TM}}

\newcommand{\calD}{\mathcal{D}}

\newcommand{\calE}{\mathcal{E}}
\newcommand{\calX}{\mathcal{X}}
\newcommand{\calT}{\mathcal{T}}
\newcommand{\calF}{\mathcal{F}}

\newcommand{\calR}{\mathcal{R}}

\newcommand{\calK}{\mathcal{K}}

\newcommand{\cvectwo}[2]{\begin{bmatrix} #1 \\ #2 \end{bmatrix}}

\newcommand{\p}{\partial}

\newcommand{\Otilde}{\tilde{O}}

\algrenewcommand{\algorithmiccomment}[1]{\hskip5em\# #1}

\title{Learning Stability Certificates from Data}

\author[2]{Nicholas M.\ Boffi\thanks{
Both authors contributed equally.
Work done while N.\ M.\ Boffi was interning at Google Brain Robotics.}}
\author[1]{Stephen Tu$^*$}
\author[3]{Nikolai Matni}
\author[4,1]{\authorcr Jean-Jacques E.\ Slotine}
\author[1]{Vikas Sindhwani}

\affil[1]{Google Brain Robotics}
\affil[2]{John A.\ Paulson School of Engineering and Applied Sciences, Harvard University}
\affil[3]{Department of Electrical and Systems Engineering, University of Pennsylvania}
\affil[4]{Nonlinear Systems Laboratory, Massachusetts Institute of Technology}

\date{August 14, 2020, Revised: \today}

\begin{document}
\maketitle


\begin{abstract}
Many existing tools in nonlinear control theory
for establishing stability or safety of a dynamical system
can be distilled to the construction of a
\emph{certificate function} that guarantees a desired property.
However, algorithms for synthesizing certificate functions
typically require a closed-form analytical expression
of the underlying dynamics, which rules out their use
on many modern robotic platforms.
To circumvent this issue, we develop algorithms for learning certificate functions only
from trajectory data. 
We establish bounds on the generalization error -- 
the probability that a certificate will not
certify a new, unseen trajectory -- when learning
from trajectories,
and we convert such generalization error bounds into global stability guarantees.
We demonstrate empirically that 
certificates for complex dynamics can be efficiently learned, and 
that the learned certificates can be used for downstream tasks such as adaptive control.
\end{abstract}

\section{Introduction}
\label{sec:introduction}

A fundamental barrier to widespread deployment of
reinforcement learning policies on real robots is the lack of formal
safety and stability guarantees. While much research
has focused on how to train control
policies for complex systems,
considerably less emphasis has been placed on verifying stability for the resulting  closed-loop system.
Without any a-priori guarantees, practitioners will be
hesitant to deploy learned solutions in the real world regardless
of performance in simulation.

Many powerful tools have been developed in nonlinear control theory to address the safety and stability of systems with known dynamics. The most well-known technique is the construction of a Lyapunov function~\cite{lyap_original, slot_li_book} to demonstrate asymptotic stability of a system with respect to an equilibrium point.
Similarly, barrier functions~\cite{blanchini99setinvariance,ames19cbf, brett_barrier} are used to show set-invariance, which
has been widely used in safety-critical applications to
prove that a system does not exit a desired safe set.
Contraction analysis~\cite{lohmiller98contraction} provides an alternative view of stability, applicable to  many problems in nonlinear control and robotics, by considering the convergence of trajectories towards each other rather than to an equilibrium point.
The unifying theme among these tools is the construction of a
\emph{certificate function} (the Lyapunov/barrier function or contraction metric) that proves a given desirable property
for the system of interest. These certificates have
strong converse results~\cite{kellett15converse,giesl15converse, converse_contr}, which imply the existence of a certificate
function if the desired property does hold,
and can also be used for
controller synthesis~\cite{sontag89universal,ames19cbf,krstic95clf}.

The main obstacle for producing certificate functions in modern
robotics and reinforcement learning is that existing synthesis and verification tools
such as sum-of-squares (SOS) optimization~\cite{ahmadi17sos}
or SMT solvers~\cite{gao13dreal}
typically assume the dynamics can be written down analytically in closed form.
Furthermore,
the functions of interest are often constrained to lie in
restrictive classes such as
polynomial basis functions of fixed degree.
This presents a serious hurdle
in modern robotics, where (a) sophisticated physics simulators are widely used to model complex environments
and (b) control policies are often represented
with complex deep neural networks.
Finally, both SOS optimization and formal verification
tools are computationally intensive,
thus limiting their applicability.

To avoid these limitations, recent approaches have proposed to treat certificate
synthesis as a machine learning problem, and train powerful function
approximators such as deep neural networks and reproducing kernel Hilbert space (RKHS) predictors on trajectory data collected from a dynamical system \cite{richards18lyapunov,manek19learningstable,taylor19cbf,robey20cbf,jin20neural,singh19learning}.
The general strategy is to enforce the desired certificate condition
(e.g. the Lie derivative of a function $V$ should be negative)
along collected samples. Empirically, this has been shown to be
quite effective, and the learned certificate often generalizes well
outside of the training data. However, a deeper theoretical understanding of when and why this approach works is missing.

\paragraph{Contributions.}

\begin{wrapfigure}[16]{r}{0.34\textwidth}
  \vspace{-20pt}
  \begin{center}
    \includegraphics[width=0.34\textwidth]{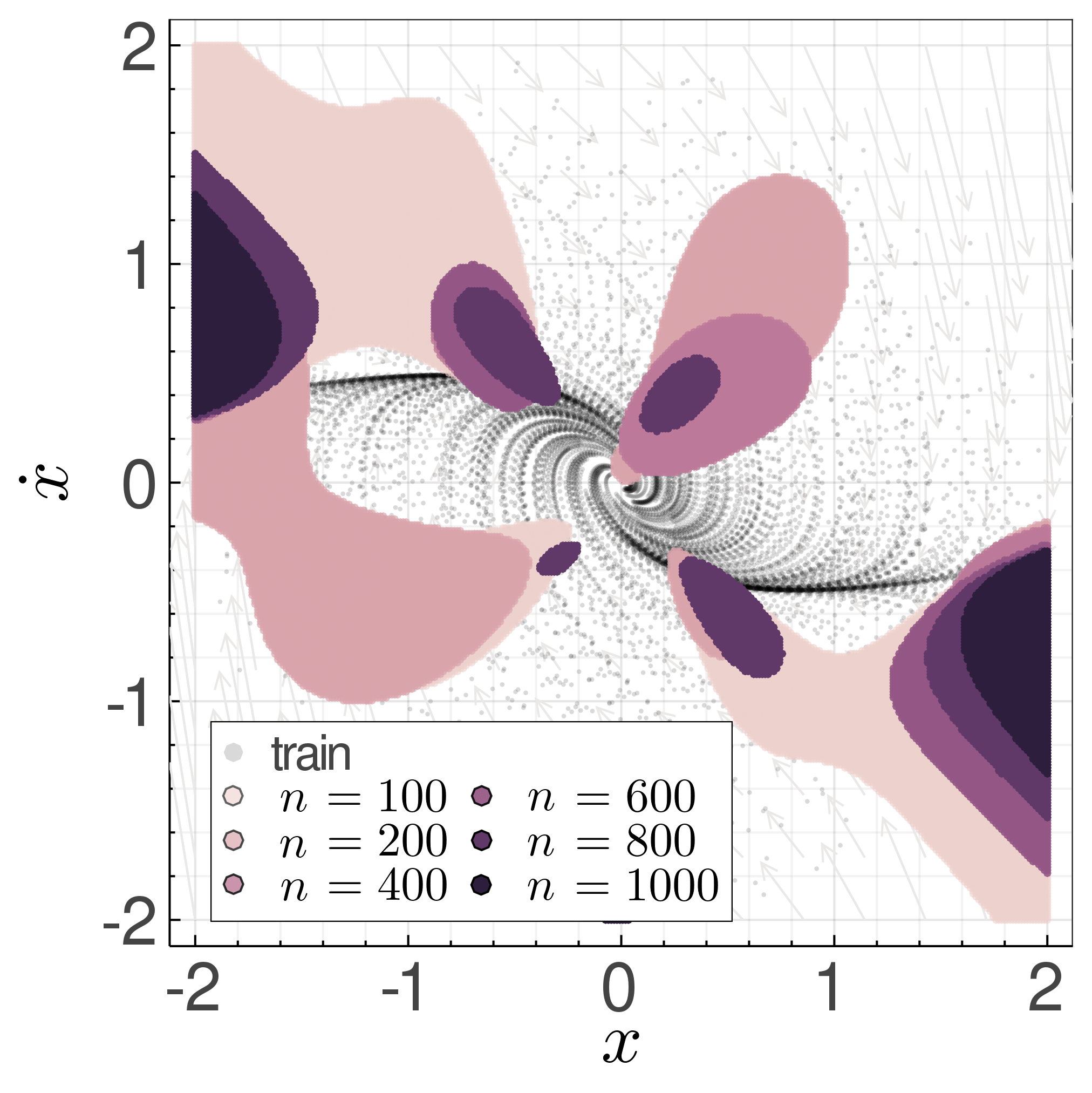}
  \end{center}
  \vspace{-10pt}
  \caption{Violation sets for contraction metric learning on the damped Van der Pol oscillator.
  }
  \label{fig:vdp_pp}
\end{wrapfigure}
Consider Figure~\ref{fig:vdp_pp}, where
a contraction metric, which certifies pairwise convergence of
trajectories, is learned from rollouts of a
damped Van der Pol oscillator. Regions of the state space
for which the learned metric is \emph{not} contracting are shown
as a function of the number of trajectories $n$.
While the size of the violating regions appears to shrink as
$n$ increases,
Figure~\ref{fig:vdp_pp} raises many questions. How much data does one need to collect so that the violating
regions cover at most a prescribed fraction of the relevant state space?
Is the learning consistent, i.e.~do the regions vanish as $n \to \infty$?

In this paper, we show that learning is indeed consistent.
To this end, we compute upper bounds on the volume of the violating
regions which tend to zero as $n$ grows.
We do this in two steps.
First, we formulate a general optimization framework
that encompasses learning many existing certificate functions, and
use statistical learning theory to prove a fast
$\Otilde(k/n)$ rate on the generalization error -- the probability the learned certificate will not certify a new, unseen trajectory --  where $k$ is the effective number of parameters of the
function class for the certificate.
We then translate bounds on the generalization error
into non-probabilistic bounds on the volume of the
violating regions.
We conclude with experiments, which show
that certificates can be efficiently learned from
trajectories, and that the learned certificates
can perform downstream tasks such as adaptive control against unknown disturbances.

\section{Related Work}
\label{sec:related}

Prior research generally focuses on learning certificates for a fixed system from trajectories, or on using certificate conditions as regularizers when learning models for control.

\paragraph{Learning Lyapunov functions from data.}

\citet{giesl20lyapunov} propose to learn a Lyapunov function from noisy trajectories using a specific reproducing kernel.
Their algorithm first fits a dynamics model from data, and then
uses interpolation to construct a Lyapunov function from the learned model.
The authors prove $L_\infty$ convergence results on the Lie derivative
of the constructed Lyapunov function compared to the ground truth,
with rates depending on a dense cover of the state space.

Our work circumvents this two-step identification procedure by directly analyzing the generalization error of a Lyapunov function learned by enforcing derivative conditions along the training data. 
Many other authors have proposed similar approaches. \citet{kenanian19switchedlinear} show how to estimate the joint spectral radius of a switched linear system 
by learning a common quadratic Lyapunov function directly
from data. Their analysis heavily exploits properties of linear systems.
\citet{chen20PWALyap} study how to learn a quadratic Lyapunov
function for piecewise affine systems in feedback with
a neural network controller.  
\citet{richards18lyapunov} use a sum-of-squares neural network
representation to learn the largest region of attraction of a nonlinear
system. \citet{manek19learningstable} jointly
train a neural network model and Lyapunov function.
Neither \citet{richards18lyapunov}
nor \citet{manek19learningstable} provide formal guarantees that the learned Lyapunov function will generalize to new trajectories.
Both \citet{chang19neuralcontrol} and \citet{ravanbakhsh19lyapunov} propose to use ideas from formal verification to falsify the validity of a learned candidate Lyapunov function. A significant limitation is the requirement of access to the true dynamics.

In many of these works, the Lie derivative constraint that defines a Lyapunov function is relaxed to a soft constraint, so that first-order gradient methods can be used for optimization. We note that our generalization analysis can be modified to handle soft constraints in a straightforward manner.

\paragraph{Learning barrier functions from data.}

Barrier functions are relaxations of Lyapunov functions
that demonstrate invariance of a subset of the state space.
Recently, many authors have proposed to use and learn barrier functions
from data for safety-critical applications. 
\citet{taylor19cbf} assume a control barrier function (CBF) is valid
for both a nominal and unknown system model, and use the CBF to guide safe learning of the unknown system dynamics.  More closely related to our work,  
\citet{robey20cbf} learn a CBF for a known nonlinear dynamical system from expert demonstrations, and use Lipschitz
arguments to extend the validity of the CBF beyond the training data.
\citet{jin20neural} propose to jointly learn a Lyapunov, barrier, and
a policy function from data. They also prove validity of the learned
certificates using Lipschitz arguments.

\paragraph{Learning contracting vector fields and contraction metrics from data.}

The literature on learning contraction metrics~\cite{lohmiller98contraction}
from data is more sparse.
In an imitation learning context, \citet{sindhwani18vectorfields}
propose to learn a vector field from demonstrations
that satisfies contraction in the identity metric. The authors parameterize the vector field as a vector-valued reproducing kernel. 
\citet{elkhadir19teleop} also learn a vector field from demonstrations by using sum-of-squares to enforce contraction.
They argue by smoothness considerations that the
learned vector field actually contracts in a tube around the
demonstration trajectories.
We note that in both these works, the metric is held fixed and is assumed to be known.
\citet{singh19learning} jointly learn a model and a control contraction metric~\citep{ccm_orig, sumeet_icra} from data, and show empirically that using contraction as a regularizer in model
learning can lead to better sample efficiency when learning to control.
We leave studying the generalization properties of jointly learning 
an explicit model and a contraction metric to future work.

\paragraph{Statistical bounds in optimization and control.}
Our generalization bounds are similar in spirit to
those provided for
random convex programs (RCPs)~\cite{campi08rcp,calafiore10rcp}. Random convex programming is concerned with approximating solutions to convex programs with an infinite number of constraints. Such infinitely-constrained problems
are approximated by drawing $n$ i.i.d.\ samples
from a distribution $\nu$ over the constraint parameters and enforcing constraints on samples. 
One can then show that the probability
that a new sample from $\nu$ violates the constraint
for the approximate solution scales as $O(d/n)$ where $d$
is the number of decision variables. 
Our results can be viewed as generalizing these bounds beyond convex
programs, though our constants are less sharp.
In our experiments, we use the RCP bound for numerically computing generalization bounds when the problem is convex.

\section{Learning Certificates Framework}
\label{sec:local}

\subsection{Problem Statement}

We assume the underlying dynamical system is given by a continuous-time autonomous system of the form
$\dot{x} = f(x)$, where $f$ is continuous, unknown, and
the state $x \in \R^p$ is fully observed.
Let $X \subseteq \R^p$ be a compact set
and let $T \subseteq \R_+$ be the maximal interval
starting at zero for which a unique solution $\varphi_t(\xi)$
exists for all initial conditions $\xi \in X$ and $t \in T$.
We assume access to sample trajectories generated from random initial conditions.
Specifically, let
$\calD$ denote a distribution over $X$,
and let $\xi_1, ..., \xi_n$ be $n$ i.i.d.\ samples from $\calD$.
We are given access to the $n$ trajectories $\{ \varphi_t(\xi_i) \}_{i=1, ..., n, t \in T}$ .
For simplicity of exposition, we
assume that we can exactly differentiate the trajectories $\varphi_t(\xi)$ with respect to
time. In our experiments, we compute $\dot{x}$ numerically.

Let $\calV$ be a space of continuously differentiable
functions $V : \R^p \mapsto \R^q$.
Let $h : \R^p \times \R^p \times \R^q \times \R^{q \times p} \mapsto \R$
be a fixed and known continuous function. Our goal is to choose a $V \in \calV$
such that
\begin{align}
    h\left(\varphi_t(\xi), \dot{\varphi}_t(\xi), V(\varphi_t(\xi)),  \frac{\p V}{\p x}(\varphi_t(\xi))\right) \leq 0 \:\: \forall \xi \in X, \: t \in T \:. \label{eq:h_decrease}
\end{align}
As we describe below, through suitable choices of the function $h$, equation \eqref{eq:h_decrease} can be used to enforce various defining conditions for certificates such as Lyapunov functions and contraction metrics.
We note that our framework can be modified to allow for more
derivatives of $V$, including higher order
derivatives and also time derivatives for handling time-varying dynamics.

We study the following optimization problem
for searching for a solution to \eqref{eq:h_decrease}:
\begin{align}
    &\operatorname{find}_{V \in \calV}~\mathrm{s.t.} ~~h\left(\varphi_t(\xi_i), \dot{\varphi}_t(\xi_i), V(\varphi_t(\xi_i)), \frac{\p V}{\p x} (\varphi_t(\xi_i))\right) \leq -\gamma \:, \:\: i=1, ..., n,  \: t \in T \:. \label{eq:V_opt}
\end{align}
Here, $\gamma > 0$ is a positive margin value which will allow us to generalize
the behavior of $V$ on $h$ outside of the sampled data.
In practice, we often solve \eqref{eq:V_opt}
with a cost term on $V$ such as its norm.
Let $\hat{V}_n \in \calV$ denote a solution to \eqref{eq:V_opt},
assuming one exists.
We quantify the \emph{generalization} of $\hat{V}_n$ by the probability of
violation over trajectories starting from $\xi \sim \calD$:
\begin{align}
    \mathsf{err}(\hat{V}_n) := \Pr_{\xi \sim \calD}\left\{ \max_{t \in T} h\left(\varphi_t(\xi), \dot{\varphi}_t(\xi), V(\varphi_t(\xi)), \frac{\p V}{\p x} (\varphi_t(\xi))\right) > 0 \right\} \:. \label{eq:err}
\end{align}
In Section~\ref{sec:local:bounds}, we prove $O( k \cdot \mathrm{polylog}(n) /n)$ decay rates for $\mathsf{err}(\hat{V}_n)$ for various parametric and non-parametric
function classes $\calV$, where $k$ denotes the effective number of parameters of the class
$\calV$.
In Section~\ref{sec:global}, we show how $\mathsf{err}(\hat{V}_n) \leq \varepsilon$ bounds translate into
global, non-probabilistic results. Before we state our main results, we instantiate our framework for two key certificate functions.

\subsubsection{Lyapunov stability analysis}

Let zero be an equilibrium point for $\dot{x} = f(x)$.
Let $D \subseteq \R^p$ be an open set containing the origin.
A Lyapunov function $V : \R^p \mapsto \R$ is a locally positive definite function
such that $V(0) = 0$, $V(x) > 0$ for $x \in D \setminus \{0\}$,
and $\ip{\nabla V(x)}{f(x)} < 0$ for $x \in D \setminus \{0\}$.
It is well known (see e.g.~\citet{slot_li_book}) that the existence of such a Lyapunov function $V$ proves the
local asymptotic stability of the origin.
Our framework can be used to learn a Lyapunov function from stable trajectories
by taking
$h(x, \dot{x}, V(x), \nabla V(x)) = \ip{\nabla V(x)}{\dot{x}} + \alpha(V(x))$.
Here, $\alpha : \R \mapsto \R$ is a class $\calK$ function, i.e. a continuous, strictly increasing function
satisfying $\alpha(0) = 0$.

\subsubsection{Contraction metrics}
A system is said to be contracting in a region $D$ with
rate $\alpha$ if there exists a
uniformly positive definite Riemannian metric $M(x)$ such that
$\frac{\p f}{\p x}(x)^\T M(x) + M(x) \frac{\p f}{\p x}(x) + \dot{M}(x) \preceq - 2 \alpha M(x)$
for $x \in D$~\cite{lohmiller98contraction}.
Given knowledge of $\frac{\p f}{\p x}$,
this condition fits into our framework by taking
$h\left(x, \dot{x}, M(x), \frac{\p f}{\p x}(x)\right) = \lambda_{\max}\left(\frac{\p f}{\p x}(x)^\T M(x) + M(x) \frac{\p f}{\p x}(x) + \dot{M}(x) + 2 \alpha M(x)\right)$.

Without knowledge of $\frac{\p f}{\p x}$, it is not immediately clear how to evaluate $h$ from trajectories.
Instead, we leverage results from \citet{forni13differentiallyapunov}, who reformulate contraction in terms of Lyapunov theory.
Consider a
candidate differential Lyapunov function $V(x, \delta x) = \delta x^\T M(x) \delta x$ for the \textit{prolongated system}
$\cvectwo{\dot{x}}{\delta\dot{x}} = \cvectwo{f(x)}{\frac{\p f}{\p x}(x) \delta x}$
defined on the tangent bundle $\calT D = \cup_{x\in D}\{x\}\times\calT_{x}D \simeq D\times \mathbb{R}^{p}$.
The contraction condition is equivalent to:
\begin{align}
    \ip{\nabla_x V(x, \delta x)}{f(x)} + \ip{\nabla_{\delta x} V(x, \delta x)}{ \frac{\p f}{\p x}(x) \delta x} \leq - \alpha V(x, \delta x) \:\:\forall x \in D \:, \delta x \in \R^p \:. \label{eq:diff_lyap_condition}
\end{align}

We can enforce \eqref{eq:diff_lyap_condition} by directly sampling trajectories on $\calT D$, by exploiting that the \textit{variational dynamics} obeyed by $\delta x(t)$ is identical to the local linearization of $f$ around $x(t)$.
Specifically, we sample pairs of initial conditions $x^{(1)}_0$ and $x^{(2)}_0 = x^{(1)}_0 + \delta x_0$ for some small perturbation $\delta x_0$. Numerical differentiation of $x^{(1)}(t)$ and $\delta x(t) = x^{(1)}(t) - x^{(2)}(t)$ provides access to $\dot{x}^{(1)} = f(x^{(1)})$ and $\delta\dot{x}(t) = \frac{\p f}{\p x}\delta x(t)$, which then
allows us to evaluate \eqref{eq:diff_lyap_condition} along system trajectories.

\section{Generalization Error Results}
\label{sec:local:bounds}

We first define the notion of stability we will assume.
Recall that $X$ is the set containing sample initial conditions,
and $T$ is the interval over which our trajectories evolve.
\begin{assume}[Stability in the sense of Lyapunov]
\label{assume:stability}
We assume there exists a compact set $S \subseteq \R^p$
such that $\varphi_t(\xi) \in S$ for all $\xi \in X$, $t \in T$. Let the constant $B_S := \sup_{x \in S} \norm{x}$.
\end{assume}
Note that contraction implies Assumption~\ref{assume:stability}, so that contracting systems are also covered in this setting. Next, we make some regularity assumptions
on the function class $\calV$.
\begin{assume}[Uniform boundedness of $\calV$]
\label{assume:regularity}
We assume there exist finite constants
$B_V$, $B_{\nabla V}$ such that
$\sup_{V \in \calV} \sup_{x \in S} \norm{V(x)} \leq B_V$
and $\sup_{V \in \calV} \sup_{x \in S} \bignorm{\frac{\p V}{\p x}(x)} \leq B_{\nabla V}$.
\end{assume}
Given Assumptions~\ref{assume:stability}--\ref{assume:regularity}, we define
$B_h$ (resp. $L_h$) to be an upper bound on
$\abs{h(x, f(x), V, \frac{\p V}{\p x})}$
(resp. the
Lipschitz constant of $(V, \frac{\p V}{\p X}) \mapsto h(x, f(x), V, \frac{\p V}{\p X})$)
over $x\in S$, $\abs{V} \leq B_V$ and
$\norm{\frac{\p V}{\p x}} \leq B_{\nabla V}$.
Note that both $B_h$ and $L_h$ are guaranteed to be finite by our assumptions.

We now introduce, with slight abuse of notation, the shorthand
$h(\xi, V)$ for $\xi \in X$, $V \in \calV$
as $h(\xi, V) := \max_{t \in T} h\left(\varphi_t(\xi), \dot{\varphi}_t(\xi), V(\varphi_t(\xi)), \frac{\p V}{\p x}(\varphi_t(\xi))\right)$.
The key insight to our analysis is the simple observation
that any feasible solution $\hat{V}_n$ to \eqref{eq:V_opt}
achieves zero empirical risk on the loss
$\hat{R}_n(V) := \frac{1}{n} \sum_{i=1}^{n} \ind_{\{ h(\xi_i, V) > - \gamma  \}}$. In particular, since $\Pr_{\xi \sim \calD}( h(\xi, \hat{V}_n) > 0 ) = \E_{\xi \sim \calD}\ind_{\{ h(\xi, \hat{V}_n) > 0  \}}$, we can use results from statistical learning
theory which give us fast rates for zero empirical risk minimizers with margin $\gamma$.
The following result is adapted from Theorem 5 of~\citet{srebro10smoothness}.
\begin{lem}
\label{lemma:fast_rate}
Fix a $\delta \in (0, 1)$.
Assume that Assumption~\ref{assume:stability}
and Assumption~\ref{assume:regularity} hold.
Suppose that the optimization problem
\eqref{eq:V_opt} is feasible and let $\hat{V}_n$ denote a solution.
The following statement holds
with probability at least $1-\delta$
over the randomness of $\xi_1, ..., \xi_n$ drawn i.i.d. from $\calD$:
\begin{align*}
    \Pr_{\xi \sim \calD}( h(\xi, \hat{V}_n) > 0 ) \leq K  \left( \frac{\log^3{n}}{\gamma^2} \calR_n^2(\calV)
    + \frac{2 \log(\log(4 B_h /\gamma)/\delta)}{n}
    \right) \:.
\end{align*}
Here, $\calR_n(\calV) := \sup_{\xi_1, ..., \xi_n \in X} \E_{\varepsilon \sim \mathrm{Unif}(\{\pm 1\}^n)}  \sup_{V \in \calV} \frac{1}{n} \bigabs{ \sum_{i=1}^{n} \varepsilon_i h(\xi_i, V) }$ is the Rademacher complexity of the function class $\calV$
and $K$ is a universal constant.
\end{lem}
Lemma~\ref{lemma:fast_rate} reduces bounding
$\mathsf{err}(\hat{V}_n)$ to bounding the
Rademacher complexity $\calR_n(\calV)$.
Define the norm $\norm{\cdot}_{\calV}$ on
$\calV$ as $\norm{V}_{\calV} := \sup_{x \in S} \bignorm{ \cvectwo{ V(x)}{\frac{\p V}{\p x}(x)}}$.
By Assumptions \ref{assume:stability}--\ref{assume:regularity}
and Dudley's entropy inequality~\cite{wainwright19book}, we can bound $\calR_n(\calV)$ by the
estimate
$ \calR_n(\calV) \leq \frac{24 L_h}{\sqrt{n}} \int_0^\infty \sqrt{\log N(\varepsilon; \calV, \norm{\cdot}_{\calV})} \: d\varepsilon $.
Here, $N(\varepsilon; \calV, \norm{\cdot}_{\calV})$
is the covering number of $\calV$ at resolution $\varepsilon$
in the $\norm{\cdot}_{\calV}$-norm.
We use this strategy to obtain generalization bounds for
\eqref{eq:V_opt} over various representations.
For ease of exposition we assume that $q = 1$,
i.e. $V : \R^p \mapsto \R$. The extension to $q > 1$
is straightforward.

\subsection{Lipschitz parametric function classes}

We consider the following parametric representation:
\begin{align}
    \calV = \{ V_\theta(\cdot) = g(x, \theta) : \theta \in \R^k \:, \norm{\theta} \leq B_\theta \} \:. \label{eq:lipschitz_parametric}
\end{align}
We assume $g : \R^p \times \R^k \mapsto \R$ is
twice continuously differentiable,
which implies that $\calV$ satisfies
Assumption~\ref{assume:regularity}.
The parameterization \eqref{eq:lipschitz_parametric} is very
general and encompasses function classes such as
neural networks with differentiable activation functions.
Furthermore,
Dudley's estimate combined with a volume
comparison argument yields
$\calR_n(\calV)^2 \leq O(k/n)$,
which implies the following result.
\begin{thm}
\label{thm:gen_bound_linear}
Under Assumption~\ref{assume:stability},
if problem \eqref{eq:V_opt} over the
parametric function class \eqref{eq:lipschitz_parametric} is feasible, then
any solution $\hat{V}_n$ satisfies with probability
at least $1-\delta$ over $\xi_1, ..., \xi_n$:
\begin{align}
    \mathsf{err}(\hat{V}_n) \leq O(1) \left( B_\theta^2(L_g + L_{\nabla g})^2 L_h^2 \frac{k \log^3{n}}{\gamma^2 n} +  \frac{\log(\log(B_h/\gamma)/\delta)}{n} \right) \:.
\end{align}
Here, $L_g := \sup_{x \in S, \norm{\theta} \leq B_\theta} \norm{ \nabla_\theta g(x, \theta)}$
and $L_{\nabla g} := \sup_{x \in S, \norm{\theta} \leq B_\theta} \norm{ \frac{\p^2 g}{\p \theta \p x}(x, \theta)}$.
\end{thm}
Often times \eqref{eq:lipschitz_parametric}
is more structured. For instance, in sum-of-squares (SOS) optimization, we have:
\begin{align}
    \calV = \left\{ V_Q(x) = m(x)^\T Q m(x) : Q \in \R^{d \times d}, Q = Q^\T \succeq 0, \norm{Q}_F \leq B_Q \right\} \:, \label{eq:sos_parametric}
\end{align}
where $m : \R^p \mapsto \R^d$ is a monomial feature map.
Note that $\eqref{eq:sos_parametric}$
is an instance of $\eqref{eq:lipschitz_parametric}$
with $k = d(d+1)/2$.
Hence Theorem~\ref{thm:gen_bound_linear}
implies a bound of the form $\mathsf{err}(\hat{V}_n) \leq \Otilde( d^2 /n )$.
However, we can actually use the matrix structure
of \eqref{eq:sos_parametric} to sharpen the
bound to $\mathsf{err}(\hat{V}_n) \leq \Otilde( d/n )$
by a more careful estimate of $\calR_n(\calV)$
using the dual Sudakov inequality~\cite{vershynin09gfa}.
\begin{thm}
\label{thm:gen_bound_sos}
Under Assumption~\ref{assume:stability},
if problem \eqref{eq:V_opt} over the
parametric linear function class \eqref{eq:sos_parametric} is feasible, then
any solution $\hat{V}_n$ satisfies with probability
at least $1-\delta$ over $\xi_1, ..., \xi_n$:
\begin{align}
    \mathsf{err}(\hat{V}_n) \leq O(1) \left( B_Q^2( B_m^2 + B_{Dm} B_m)^2 L_h^2 \frac{d \log^2{d} \log^3{n}}{\gamma^2 n} +  \frac{\log(\log(B_h/\gamma)/\delta)}{n} \right) \:.
\end{align}
Here, $B_m := \sup_{x \in S} \norm{m(x)}$
and $B_{Dm} := \sup_{x \in S} \norm{\frac{\p m}{\p x}(x)}$.
\end{thm}

In general, using prior knowledge about the system
to add more structure and reduce the number of parameters
of the certificate function
(e.g. diagonal contraction metrics for positive systems) yields
better generalization bounds.

\subsection{Reproducing kernel Hilbert space function classes}

We now consider the following non-parametric function class:
\begin{align}
    \calV = \left\{ V_\alpha(\cdot) = \int_{\Theta} \alpha(\theta) \phi(\cdot; \theta) \: d\theta : \norm{V_\alpha}_\nu := \sup_{\theta \in \Theta} \bigabs{\frac{\alpha(\theta)}{\nu(\theta)}} \leq B_\alpha \right\} \:. \label{eq:mixture_rkhs}
\end{align}
Here, $\phi(\cdot; \theta)$ is a nonlinear function and $\nu$ is a probability distribution over $\Theta$.
This function class is a subset of the
reproducing kernel Hilbert space (RKHS)
defined by the kernel
$k(x, y) = \int_{\Theta} \phi(x; \theta) \phi(y; \theta) \nu(\theta) \: d\theta$, and
is dense in the RKHS as $B_\alpha \to \infty$~ \cite{rahimi08uniform}.
We further assume that $\phi(x; \theta)$ is of the
form $\phi(x; \theta) = \phi(\ip{x}{w} + b)$ with $\phi$ differentiable and
$\theta = (w, b)$.
RKHSs of this type often arise naturally.
For instance, Bochner's theorem~\cite{rahimi07randomfeatures} states that every translation invariant kernel
can be expressed in this form.
\begin{thm}
\label{thm:gen_bound_rkhs}
Suppose that $\abs{\phi} \leq 1$, $\phi$ is $L_\phi$-Lipschitz,
$\phi$ is differentiable, $\phi'$ is $L_{\phi'}$-Lipschitz,
and that $B_\theta := \sup_{\theta \in \Theta} \norm{\theta}$ is finite.
Under Assumption~\ref{assume:stability},
if problem \eqref{eq:V_opt} over the
non-parametric class \eqref{eq:mixture_rkhs} is feasible, then
any solution $\hat{V}_n$ satisfies with probability
at least $1-\delta$ over $\xi_1, ..., \xi_n$:
\begin{align*}
    \mathsf{err}(\hat{V}_n) \leq O(1) \left( B_\alpha^2 (1 + B_\theta L_\phi)^2  L_h^2  \frac{\kappa \log^3{n}}{\gamma^2 n} + \frac{\log(\log(1+B_h/\gamma)/\delta)}{n} + 1/n^2  \right) \:,
\end{align*}
where $\kappa := \frac{B_\alpha^2 L_h^2}{\gamma^2} \left( (1 + B_\theta L_\phi)^2 \log{n} + B_\theta^2 (B_S + 1)^2 (L_\phi + B_\theta L_{\phi'} )^2 p \right)$.
\end{thm}

\section{Global Stability Results}
\label{sec:global}

In this section, we show how the
bounds from Section~\ref{sec:local:bounds}
can be translated into global results
for the learned certificate functions.
To facilitate our analysis, we assume the dynamics is incrementally stable. Incremental
stability is implied by contraction,
but is stronger than Lyapunov stability.
Before stating the assumption,
we say that $\beta(s, t)$ is a class $\mathcal{KL}$ function
if for every $t$ the map $s \mapsto \beta(s, t)$ is
a class $\mathcal{K}$ function and
for every $s$ the map $t \mapsto \beta(s, t)$ is
continuous and non-increasing.
\begin{assume}[Incremental stability, c.f.~\citet{raginsky}]
\label{assume:incr_stability}
There exists a class $\mathcal{KL}$ function $\beta$
such that for all $\xi_1, \xi_2 \in X$,
$\norm{\varphi_t(\xi_1) - \varphi_t(\xi_2)} \leq \beta(\norm{\xi_1 - \xi_2}, t)$
for all $t \in T$.
\end{assume}

With Assumption~\ref{assume:incr_stability} in hand,
we are ready to state a result regarding learned
Lyapunov functions.
For what follows, let $\mathbb{B}_2^p(r)$ denote the closed $\ell_2$-ball in $\R^p$ of
radius $r$, $\mathbb{S}^{p-1}$ denote the sphere in $\R^p$, and
$\leb(\cdot)$ denote the Lebesgue measure on $\R^p$.
\begin{thm}
\label{thm:global_lyap}
Suppose the system satisfies Assumption~\ref{assume:incr_stability},
and suppose the set $X$ is full-dimensional and compact.
Define the set $S := \cup_{t \in T} \varphi_t(X)$.
Let $V : S \mapsto \R$ be a twice-differentiable positive definite function
satisfying $V(x) \geq \mu \norm{x}^2$ for all $x \in S$.
Define the violation set $X_b$ as:
\begin{align}
    X_b := \left\{ \xi \in X : \max_{t \in T} \ip{\nabla V(\varphi_t(\xi))}{f(\varphi_t(\xi))} > \lambda V(\varphi_t(\xi)) \right\} \:.
\end{align}
Let $\nu$ denote the uniform probability measure on $X$ and suppose that
$\nu(X_b) \leq \varepsilon$.
Define the function
$q(x) := \ip{\nabla V(x)}{f(x)}$, and denote the constants $B_{\nabla q} := \sup_{x \in S} \norm{\nabla q(x)}$, $B_{\nabla V} := \sup_{x \in S} \norm{\nabla V(x)}$.
Let $r(\varepsilon) := \left( \frac{\varepsilon \leb(X)}{\leb(\mathbb{B}_2^p(1))} \right)^{1/p}$.
Then for all $\eta \in (0, 1)$:
\begin{align}
    \ip{\nabla V(x)}{f(x)} \leq - (1-\eta)\lambda V(x) \:\:\forall x \in \tilde{S} \setminus \mathbb{B}_2^p\left(0, \sqrt{ \frac{\beta(r(\varepsilon), 0)}{\eta \mu} (B_{\nabla V} + \lambda^{-1} B_{\nabla q})   }\right) \:. \label{eq:lie_derivative_ball}
\end{align}
Here, $\tilde{S} := \cup_{t \in T} \varphi_t(\tilde{X})$
with $\tilde{X} := \{ \xi \in X : \mathbb{B}_2^p(\xi, r(\varepsilon)) \subset X \}$.
Furthermore, for every $\xi \in X$,
let $u_\xi(t)$ denote the solution to the differential equation:
\begin{align}
    \dot{u}_\xi = - \lambda u_\xi + (B_{\nabla q} + \lambda B_{\nabla V}) \beta( r(\varepsilon), t) \:, \:\: u_\xi(0) = V(\xi) \:. \label{eq:V_upper_bound}
\end{align}
Then for every $\xi \in \tilde{X}$ and $t \in T$,
the inequality $V(\varphi_t(\xi)) \leq u_\xi(t)$ holds.
\end{thm}

Theorem~\ref{thm:global_lyap}
states that the learned Lyapunov function
$V$ satisfies the Lie derivative decrease condition
on all of $\tilde{S}$ except for a ball
of radius $r_b \leq O(\sqrt{\beta(r(\varepsilon), 0)})$
around the origin. Since $r(\varepsilon) \to 0$
as $\varepsilon \to 0$, Theorem~\ref{thm:global_lyap}
shows that the quality of our Lyapunov function increases as  the measure of the violation set $X_b$ decreases.
Furthermore, we can
apply the bounds in Section~\ref{sec:local:bounds}
to obtain an upper bound on the
radius $r_b$ of the ball as a function of the
number of sample trajectories.
For example, Theorem~\ref{thm:gen_bound_linear}
states that $\nu(X_b) \leq O(k/n)$ if $n$ random samples are drawn uniformly from X.
For simplicity assume $X = \mathbb{B}_2^p(1)$
and $\beta(r, 0) \leq O(r)$,
which
implies
$r_b \leq O((k/n)^{1/2p})$.
Setting $r_b \leq \zeta$
and solving for $n$, we find $n \geq \Omega( k \cdot \zeta^{-2p})$.
Such an
exponential dependence on the
dimension $p$ is unavoidable
without assuming more structure.

Equation~\eqref{eq:V_upper_bound}
yields bounds of the form
$V(\varphi_t(\xi)) \leq V(\xi) e^{-\lambda t} + O(r(\varepsilon) h_\beta(t))$,
where $h_\beta$ depends on the specific form of $\beta$.
For example if
$\beta(s, t) \leq M e^{-\alpha t} s$ for some $\alpha > \lambda$,
then $h_\beta(t) = e^{-\lambda t}$.
On the other hand,
if we have the slower rate $\beta(s, t) \leq M s/(t+1)$,
then $h_\beta(t) = t^2 e^{-\lambda t}$.

We note that Theorem~\ref{thm:global_lyap}
is conceptually similar to the results
from \citet{liu20almostlyapunov}, but
incremental stability assumption dramatically simplifies the proof and
enables us to make the constants explicit.

We now state a similar result
to Theorem~\ref{thm:global_lyap}
for metric learning.
Let $\psi_t(\cdot)$ denote the induced flow on the prolongated system $g(x, \delta x) = (f(x), \frac{\p f}{\p x}(x) \delta x)$
and $\theta_t(\delta \xi; \xi)$ denote the second element of
$\psi_t(\xi, \delta \xi)$.
Further let $\zeta_p(r)$ be the Haar measure of a spherical cap in $\mathbb{S}^{p-1}$
with arc length $r$.
\begin{thm}
\label{thm:global_diff_lyap}
Fix an $\eta \in (0, 1)$.
Suppose that $X \subseteq \R^p$ is full-dimensional and $p \geq 2$.
Let $\dot{x} = f(x)$ be contracting in the metric $M_\star(x)$ with rate $\lambda > 0$. Assume that $m I \preceq M_\star(x) \preceq L I$.  Let $V(x, \delta x): \calT S\mapsto \R$ be of the form $V(x, \delta x) = \delta x^\T M(x) \delta x$
for some positive definite matrix function $M(x)$
satisfying $M(x) \succeq \mu I$.
Define the violation set $Z_b$ as:
\begin{align}
    Z_b := \left\{ (\xi, \delta \xi) \in X \times \mathbb{S}^{p-1} : \max_{t \in T} \ip{\nabla V(\psi_t(\xi, \delta \xi))}{g(\psi_t(\xi, \delta \xi))} > \lambda V(\psi_t(\xi, \delta \xi)) \right\} \:.
\end{align}
Suppose that $\nu(Z_b) \leq \varepsilon$,
where $\nu$ is the uniform probability measure on $X \times \mathbb{S}^{p-1}$.
Define
$r(\varepsilon) := \sup\left\{ r > 0 : r^p \zeta_p(r) \leq \frac{\varepsilon \leb(X)}{\leb(\mathbb{B}_2^p(1))} \right\}$,
and let the radius $r_b := \sqrt{ r(\varepsilon) B_H (B_{\nabla q} + \lambda B_{\nabla V})(L/m)^{3/2} / (\eta\lambda \mu) }$,
where $B_H := \sup_{x \in S} \left\Vert\frac{\p^2 f}{\p x^2}\right\Vert$.
Finally, define the sets $\tilde{X}_t(r_b) := \{ \xi \in \tilde{X} : \inf_{\delta \xi \in \mathbb{S}^{p-1}} \norm{\theta_t(\delta \xi; \xi)} \geq r_b \}$ for $t \in T$, with $\tilde{X} := \{ \xi \in X : \mathbb{B}_2^p(\xi, r(\varepsilon)) \subset X \}$.
Then the system will be contracting in the metric $M(x)$ at the rate $(1-\eta) \lambda$ for every $x \in \tilde{S}(r_b) := \cup_{t \in T} \varphi_t(\tilde{X}_t(r_b))$.

\end{thm}
Theorem~\ref{thm:global_diff_lyap} is illustrated in Figure~\ref{fig:vdp_pp}, which shows the structure of the violation set. Further details and exploration of the effect of $\eta$ can be found in Section~\ref{app:vdp}. In Section~\ref{sec:app:known_dynamics} we prove a
result similar to Theorem~\ref{thm:global_diff_lyap} for metric learning
with known dynamics.

\section{Learning Certificates in Practice}
\label{sec:experiments}

We empirically study the generalization behavior
of both learning Lyapunov functions
and contraction metrics
from trajectory data.
We consider Lyapunov functions parameterized by
$V(x) = x^\T (L(x)L(x)^\T + I) x$,
where $L(x)$ is the (reshaped) value of a
fully connected neural network with $\tanh$ activations
of size $p \times h \times h \times p \cdot (2 p)$,
where $p$ is the state-dimension of $x$ and
$h$ is the hidden width.
For metric learning,
we study a convex formulation via SOS programming. Each matrix element $M_{ij}(x) = \ip{w_{ij}}{\phi(x)}$ is given by a polynomial where $w_{ij}$ are the learned weights and $\phi(x)$ is a feature map of monomials in the state vector. In our experiments, we numerically estimate the
generalization error of a learned certificate using a test set.
We compute an upper confidence bound (UCB) of the estimate
using the Chernoff inequality with $\delta=0.01$, as described by  \citet{langford05practical}. More experimental details are
given in the appendix.

\begin{figure}[ht]
\begin{tabular}{cc}
    \hspace{-10pt}\includegraphics[width=.49\textwidth]{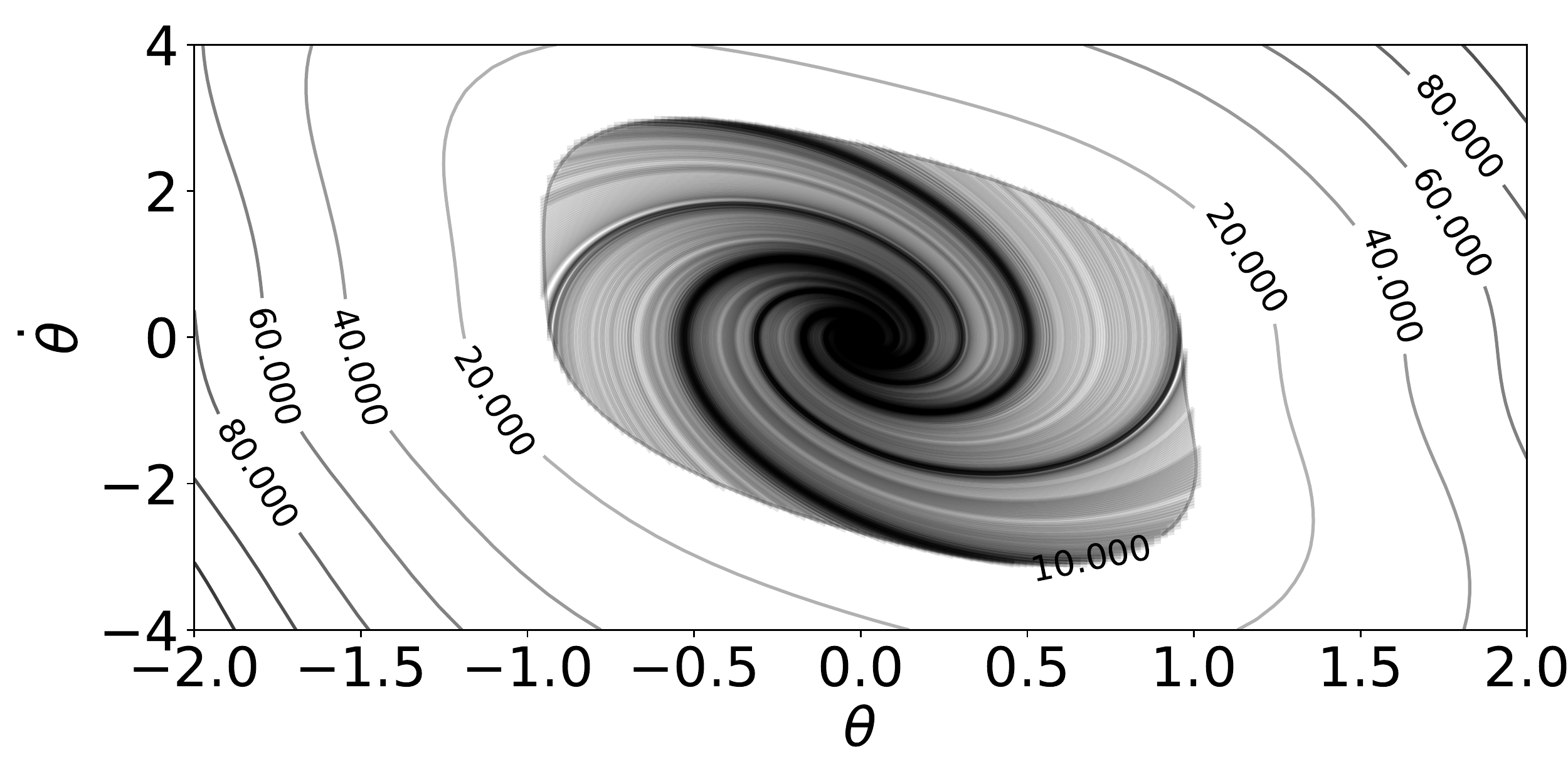} & \includegraphics[width=.49\textwidth]{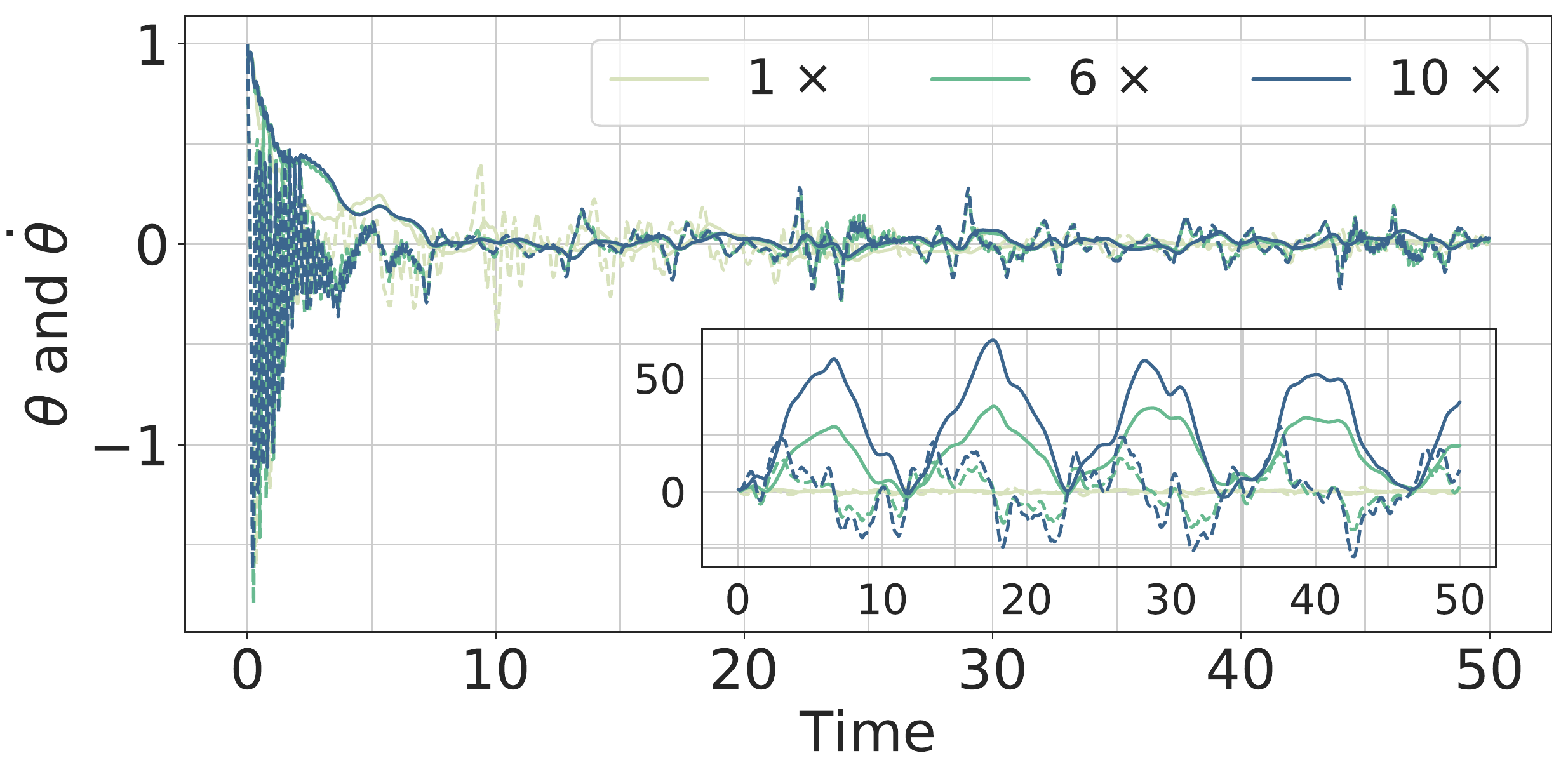} \\
    \hspace{-10pt}\includegraphics[width=.49\textwidth]{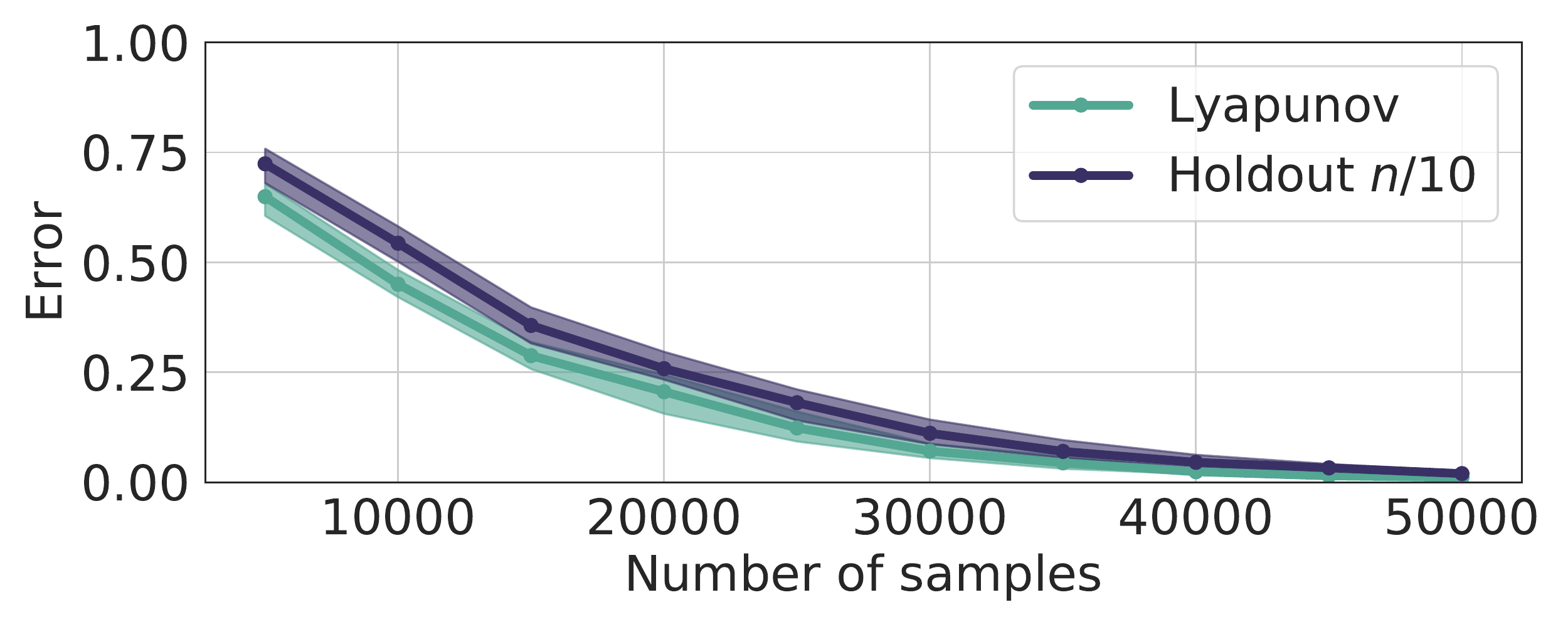} & \includegraphics[width=.49\textwidth]{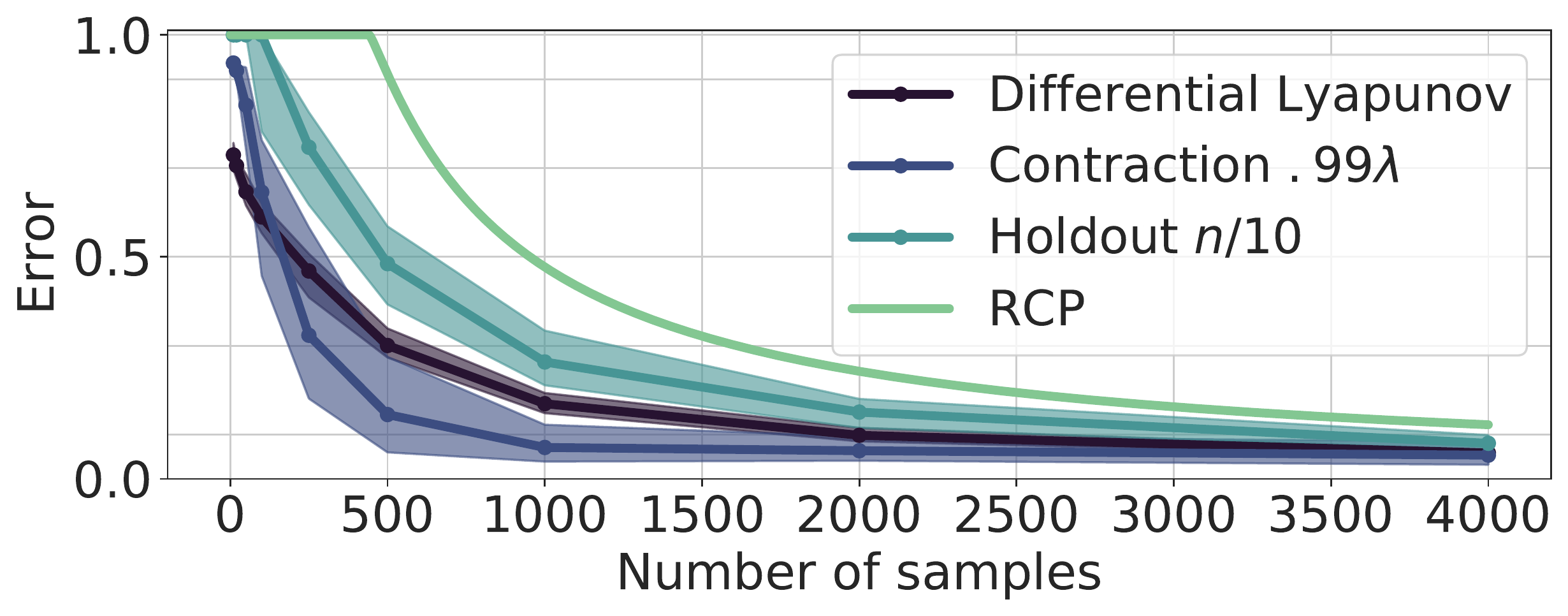}
\end{tabular}
\caption{
\textbf{(UL)} The level sets of a learned
Lyapunov function on trajectories from a
damped pendulum. Trajectories
are initialized along the $10$ level set
and rolled out to demonstrate set invariance.
\textbf{(UR)} Adaptive control of a damped
pendulum driven by a random sinusoidal input.
$\theta$ is shown in solid and $\dot{\theta}$ is shown dashed.
Color indicates strength of the input.
Main figure shows performance using adaptation via the learned Lyapunov function, while inset shows performance without adaptation.
\textbf{(LL)} Generalization error of a learned Lyapunov function
collected from trajectories of
a standing minitaur impacted with a random impulse
force and stabilized with a hand-tuned PD controller.
\textbf{(LR)} Generalization error of a learned
metric on a 6D gradient flow.
}
\label{fig:lyap_exps}
\end{figure}

\paragraph{Damped pendulum.}

We learn a Lyapunov function for the damped pendulum
from $1000$ training trajectories.
Figure~\ref{fig:lyap_exps} (UL) shows the
level sets of a typical learned Lyapunov function, where we also numerically rollout a dense set of trajectories starting
from $\{ x \in \R^2 : V_\theta(x) = 10 \}$ to check set invariance.
In Figure~\ref{fig:lyap_exps} (UR),
we add a disturbance $\ip{a}{\kappa \phi(t)}$ to the dynamics
where $a \in \R^{10}$ is unknown
and $\phi(t)$ are random sinusoids.
We use an adaptive control law~\cite{slot_li_book}
based on the learned Lyapunov function
to regulate $x \to 0$ (see the appendix for details). We vary $\kappa \in \{1, 6, 10\}$ to
study the robustness of the adaptation.
Figure~\ref{fig:lyap_exps} (UR) shows that the learned
Lyapunov function is able to provide enough information
to robustly regulate the state even as the disturbance $\kappa$
increases by a factor of $10$, whereas the system without adaptation is driven far from the origin.

\paragraph{Stable standing for quadrupeds.}

We learn a discrete-time Lyapunov function for a quadruped robot~\cite{kenneally2016design} as it recovers from external forcing. We apply a random impulse force in the $(x, y)$ plane at time $t=0$ to the Minitaur quadruped environment in PyBullet~\cite{coumans2020}, and use a hand-tuned PD
controller to return the minitaur to a standing position.
We train a discrete-time
Lyapunov function in order to handle the discontinuities in the trajectories
introduced by contact forces.

Figure~\ref{fig:lyap_exps} (LL) shows the result of this
experiment.
For the \textbf{Lyapunov} curve,
the resulting model trained on $n$ trajectories is
then validated using a $10000$ trajectory test set.
The generalization error is the ratio
of trajectories which violate the desired decrease condition
\emph{for any step $k \in \{1, ..., 199\}$}.
We run $30$ trials and plot the
10/50/90-th percentile of the generalization UCB.
With $n=50000$, the median generalization UCB
is $1.11\%$.
Since in practice a separate test set may not
be available,
we also compare
to splitting the available training data
into an actual training set of size $0.9 n$ and a
validation set of size $0.1 n$. The model is trained
on the actual training set, and a generalization
UCB is calculated from the validation set.
We run $30$ trials of this setup
and plot the 10/50/90-th percentile
in the \textbf{Holdout} curve.
After $n=50000$, the median generalization UCB is
$1.99\%$.

\paragraph{6-dimensional gradient system.}
Gradient flow has recently been explored in the context of Riemannian motion policies for robotics~\citep{rana2020euclideanizing, ratliff2018riemannian}, and converges for nonconvex losses with contracting dynamics~\citep{wensing2018convexity}.
We learn a metric for gradient flow on the nonconvex loss $\mathcal{L}(x) = \Vert x\Vert^2 + \sum_{i \neq j} x_i^2 x_j^2$ for $x \in \R^6$. Figure~\ref{fig:lyap_exps} (LR) shows the generalization error curves for the differential Lyapunov constraints.
Because the SOS program for metric learning is convex,
we can apply generalization bounds from randomized convex programming (RCP)~\citep{calafiore10rcp}. We also plot the probability that the learned metric is a true contraction metric with rate $.99\lambda$ on the test set (probability with rate $\lambda$ is low) and a generalization UCB obtained using a validation set. For each curve, we plot the 10/50/90-th percentile of the generalization UCB. As the number of samples increases, the error probability for differential Lyapunov constraints decreases, and the learned metric becomes a true metric with reduced rate with high probability. With $n=4000$, the median generalization UCBs are $5.85\%$, $8.02\%$, and $5.35\%$ for differential Lyapunov on the test set, differential Lyapunov on the validation set, and contraction with rate $.99\lambda$, respectively.

\section{Conclusion}
\label{sec:conclusion}

Our work shows that certificate functions can
be efficiently learned from data, and raises many
interesting questions for future work.
Extending the results to handle both noisy state observations and process noise in the dynamics would 
allow for learning certificates in uncertain environments.
Another interesting question is to establish bounds
for joint learning of both the unknown dynamics and
a certificate, which has shown to be effective
in practice~\cite{singh19learning,manek19learningstable}.
Finally, lower bounds on the learning
certificate problem would highlight the
amount of conservatism introduced in our results.

\section*{Acknowledgements}
The authors would like to thank Amir Ali Ahmadi,
Brett Lopez, Alexander Robey, and Sumeet Singh for providing helpful feedback.
{\small
\bibliography{paper}

\begin{thebibliography}{52}
\providecommand{\natexlab}[1]{#1}
\providecommand{\url}[1]{\texttt{#1}}
\expandafter\ifx\csname urlstyle\endcsname\relax
  \providecommand{\doi}[1]{doi: #1}\else
  \providecommand{\doi}{doi: \begingroup \urlstyle{rm}\Url}\fi

\bibitem[{Ahmadi} et~al.(2017){Ahmadi}, {Hall}, {Papachristodoulou},
  {Saunderson}, and {Zheng}]{ahmadi17sos}
A.~A. {Ahmadi}, G.~{Hall}, A.~{Papachristodoulou}, J.~{Saunderson}, and
  Y.~{Zheng}.
\newblock Improving efficiency and scalability of sum of squares optimization:
  Recent advances and limitations.
\newblock In \emph{2017 IEEE 56th Annual Conference on Decision and Control},
  2017.

\bibitem[{Ames} et~al.(2019){Ames}, {Coogan}, {Egerstedt}, {Notomista},
  {Sreenath}, and {Tabuada}]{ames19cbf}
A.~D. {Ames}, S.~{Coogan}, M.~{Egerstedt}, G.~{Notomista}, K.~{Sreenath}, and
  P.~{Tabuada}.
\newblock Control barrier functions: Theory and applications.
\newblock In \emph{2019 18th European Control Conference}, 2019.

\bibitem[Aylward et~al.(2006)Aylward, Parrilo, and Slotine]{sos_contr}
E.~M. Aylward, P.~A. Parrilo, and J.-J.~E. Slotine.
\newblock Algorithmic search for contraction metrics via sos programming.
\newblock In \emph{2006 American Control Conference}, 2006.

\bibitem[Blanchini(1999)]{blanchini99setinvariance}
F.~Blanchini.
\newblock Set invariance in control.
\newblock \emph{Automatica}, 35\penalty0 (11):\penalty0 1747--1767, 1999.

\bibitem[Calafiore(2010)]{calafiore10rcp}
G.~C. Calafiore.
\newblock Random convex programs.
\newblock \emph{SIAM Journal on Optimization}, 20\penalty0 (6):\penalty0
  3427--3464, 2010.

\bibitem[Campi and Garatti(2008)]{campi08rcp}
M.~C. Campi and S.~Garatti.
\newblock The exact feasibility of randomized solutions of uncertain convex
  programs.
\newblock \emph{SIAM Journal on Optimization}, 19\penalty0 (3):\penalty0
  1211--1230, 2008.

\bibitem[Chang et~al.(2019)Chang, Roohi, and Gao]{chang19neuralcontrol}
Y.-C. Chang, N.~Roohi, and S.~Gao.
\newblock Neural lyapunov control.
\newblock In \emph{Neural Information Processing Systems}, 2019.

\bibitem[Chen and Slotine(2012)]{converse_contr}
L.~Chen and J.~Slotine.
\newblock Notes on metrics in contraction analysis.
\newblock \emph{MIT Nonlinear Systems Laboratory Report}, \penalty0
  (NSL-121101), Nov. 2012.

\bibitem[Chen et~al.(2020)Chen, Fazlyab, Morari, Pappas, and
  Preciado]{chen20PWALyap}
S.~Chen, M.~Fazlyab, M.~Morari, G.~J. Pappas, and V.~M. Preciado.
\newblock Learning lyapunov functions for piecewise affine systems with neural
  network controllers.
\newblock \emph{arXiv:2008.06546}, 2020.

\bibitem[Coumans and Bai(2016--2020)]{coumans2020}
E.~Coumans and Y.~Bai.
\newblock Pybullet, a python module for physics simulation for games, robotics
  and machine learning.
\newblock \url{http://pybullet.org}, 2016--2020.

\bibitem[Forni and Sepulchre(2013)]{forni13differentiallyapunov}
F.~Forni and R.~Sepulchre.
\newblock A differential lyapunov framework for contraction analysis.
\newblock \emph{IEEE Transactions on Automatic Control}, 59\penalty0
  (3):\penalty0 614--628, 2013.

\bibitem[Fradkov et~al.(1999)Fradkov, Miroshnik, and Nikiforov]{fradkov99}
A.~L. Fradkov, I.~V. Miroshnik, and V.~O. Nikiforov.
\newblock \emph{Nonlinear and Adaptive Control of Complex Systems}.
\newblock 1999.

\bibitem[Gao et~al.(2013)Gao, Kong, and Clarke]{gao13dreal}
S.~Gao, S.~Kong, and E.~M. Clarke.
\newblock dreal: An smt solver for nonlinear theories over the reals.
\newblock In \emph{Automated Deduction -- CADE-24}, 2013.

\bibitem[Giesl(2015)]{giesl15converse}
P.~Giesl.
\newblock Converse theorems on contraction metrics for an equilibrium.
\newblock \emph{Journal of Mathematical Analysis and Applications},
  424\penalty0 (2):\penalty0 1380--1403, 2015.

\bibitem[Giesl et~al.(2020)Giesl, Hamzi, Rasmussen, and
  Webster]{giesl20lyapunov}
P.~Giesl, B.~Hamzi, M.~Rasmussen, and K.~Webster.
\newblock Approximation of lyapunov functions from noisy data.
\newblock \emph{Journal of Computational Dynamics}, 7\penalty0 (1):\penalty0
  57--81, 2020.

\bibitem[Hanson and Raginsky(2020)]{raginsky}
J.~Hanson and M.~Raginsky.
\newblock Universal simulation of dynamical systems by recurrent neural nets.
\newblock In \emph{Learning for Dynamics and Control}, 2020.

\bibitem[Jin et~al.(2020)Jin, Wang, Yang, and Mou]{jin20neural}
W.~Jin, Z.~Wang, Z.~Yang, and S.~Mou.
\newblock Neural certificates for safe control policies.
\newblock \emph{arXiv:2006.08465}, 2020.

\bibitem[Kappos(2001)]{kappos}
E.~Kappos.
\newblock Natural metrics on tangent bundles.
\newblock Master's thesis, Lund University, 2001.

\bibitem[Kellett(2015)]{kellett15converse}
C.~M. Kellett.
\newblock Classical converse theorems in lyapunov's second method.
\newblock \emph{Discrete \& Continuous Dynamical Systems - B}, 20\penalty0
  (8):\penalty0 2333--2360, 2015.

\bibitem[Kenanian et~al.(2019)Kenanian, Balkan, Jungers, and
  Tabuada]{kenanian19switchedlinear}
J.~Kenanian, A.~Balkan, R.~M. Jungers, and P.~Tabuada.
\newblock Data driven stability analysis of black-box switched linear systems.
\newblock \emph{Automatica}, 109:\penalty0 108533, 2019.

\bibitem[Kenneally et~al.(2016)Kenneally, De, and
  Koditschek]{kenneally2016design}
G.~Kenneally, A.~De, and D.~E. Koditschek.
\newblock Design principles for a family of direct-drive legged robots.
\newblock \emph{IEEE Robotics and Automation Letters}, 1\penalty0 (2):\penalty0
  900--907, 2016.

\bibitem[Khadir et~al.(2019)Khadir, Varley, and Sindhwani]{elkhadir19teleop}
B.~E. Khadir, J.~Varley, and V.~Sindhwani.
\newblock Teleoperator imitation with continuous-time safety.
\newblock In \emph{Robotics: Science and Systems}, 2019.

\bibitem[Krsti{\'{c}} and Kokotovi{\'{c}}(1995)]{krstic95clf}
M.~Krsti{\'{c}} and P.~V. Kokotovi{\'{c}}.
\newblock Control lyapunov functions for adaptive nonlinear stabilization.
\newblock \emph{Systems \& Control Letters}, 26\penalty0 (1):\penalty0 17--23,
  1995.

\bibitem[Langford(2005)]{langford05practical}
J.~Langford.
\newblock Tutorial on practical prediction theory for classification.
\newblock \emph{Journal of Machine Learning Research}, 6:\penalty0 273--306,
  2005.

\bibitem[Ledoux and Talagrand(1991)]{ledoux91book}
M.~Ledoux and M.~Talagrand.
\newblock \emph{Probability in Banach Spaces}.
\newblock 1991.

\bibitem[Li(2011)]{li11cap}
S.~Li.
\newblock Concise formulas for the area and volume of a hyperspherical cap.
\newblock \emph{Asian Journal of Mathematics and Statistics}, 4\penalty0
  (1):\penalty0 66--70, 2011.

\bibitem[Liu et~al.(2020)Liu, Liberzon, and Zharnitsky]{liu20almostlyapunov}
S.~Liu, D.~Liberzon, and V.~Zharnitsky.
\newblock Almost lyapunov functions for nonlinear systems.
\newblock \emph{Automatica}, 113:\penalty0 108758, 2020.

\bibitem[Lohmiller and Slotine(1998)]{lohmiller98contraction}
W.~Lohmiller and J.-J.~E. Slotine.
\newblock On contraction analysis for non-linear systems.
\newblock \emph{Automatica}, 34\penalty0 (6):\penalty0 683--696, 1998.

\bibitem[{Lopez} and {Slotine}(2021)]{brett_adapt}
B.~T. {Lopez} and J.~E. {Slotine}.
\newblock Adaptive nonlinear control with contraction metrics.
\newblock \emph{IEEE Control Systems Letters}, 5\penalty0 (1):\penalty0
  205--210, 2021.

\bibitem[{Lopez} et~al.(2021){Lopez}, {Slotine}, and {How}]{brett_barrier}
B.~T. {Lopez}, J.~E. {Slotine}, and J.~P. {How}.
\newblock Robust adaptive control barrier functions: An adaptive and
  data-driven approach to safety.
\newblock \emph{IEEE Control Systems Letters}, 5\penalty0 (3):\penalty0
  1031--1036, 2021.

\bibitem[Lyapunov(1892)]{lyap_original}
A.~M. Lyapunov.
\newblock \emph{The general problem of the stability of motion (in {R}ussian)}.
\newblock PhD thesis, University of Kharkov, 1892.

\bibitem[{Manchester} and {Slotine}(2017)]{ccm_orig}
I.~R. {Manchester} and J.~E. {Slotine}.
\newblock Control contraction metrics: Convex and intrinsic criteria for
  nonlinear feedback design.
\newblock \emph{IEEE Transactions on Automatic Control}, 62\penalty0
  (6):\penalty0 3046--3053, 2017.

\bibitem[Manek and Kolter(2019)]{manek19learningstable}
G.~Manek and J.~Z. Kolter.
\newblock Learning stable deep dynamics models.
\newblock In \emph{Neural Information Processing Systems}, 2019.

\bibitem[Maurer(2016)]{maurer16vectorcontraction}
A.~Maurer.
\newblock A vector-contraction inequality for rademacher complexities.
\newblock \emph{arXiv:1605.00251}, 2016.

\bibitem[Rahimi and Recht(2007)]{rahimi07randomfeatures}
A.~Rahimi and B.~Recht.
\newblock Random features for large-scale kernel machines.
\newblock In \emph{Neural Information Processing Systems}, 2007.

\bibitem[Rahimi and Recht(2008)]{rahimi08uniform}
A.~Rahimi and B.~Recht.
\newblock Uniform approximation of functions with random bases.
\newblock In \emph{2008 46th Annual Allerton Conference on Communication,
  Control, and Computing}, 2008.

\bibitem[Rana et~al.(2020)Rana, Li, Fox, Boots, Ramos, and
  Ratliff]{rana2020euclideanizing}
M.~A. Rana, A.~Li, D.~Fox, B.~Boots, F.~Ramos, and N.~Ratliff.
\newblock Euclideanizing flows: Diffeomorphic reduction for learning stable
  dynamical systems.
\newblock \emph{arXiv:2005.13143}, 2020.

\bibitem[Ratliff et~al.(2018)Ratliff, Issac, Kappler, Birchfield, and
  Fox]{ratliff2018riemannian}
N.~D. Ratliff, J.~Issac, D.~Kappler, S.~Birchfield, and D.~Fox.
\newblock Riemannian motion policies.
\newblock \emph{arXiv:1801.02854}, 2018.

\bibitem[Ravanbakhsh and Sankaranarayanan(2019)]{ravanbakhsh19lyapunov}
H.~Ravanbakhsh and S.~Sankaranarayanan.
\newblock Learning control lyapunov functions from counterexamples and
  demonstrations.
\newblock \emph{Autonomous Robots}, 43:\penalty0 275--307, 2019.

\bibitem[Richards et~al.(2018)Richards, Berkenkamp, and
  Krause]{richards18lyapunov}
S.~M. Richards, F.~Berkenkamp, and A.~Krause.
\newblock The lyapunov neural network: Adaptive stability certification for
  safe learning of dynamical systems.
\newblock In \emph{Conference on Robot Learning}, 2018.

\bibitem[Robey et~al.(2020)Robey, Hu, Lindemann, Zhang, Dimarogonas, Tu, and
  Matni]{robey20cbf}
A.~Robey, H.~Hu, L.~Lindemann, H.~Zhang, D.~V. Dimarogonas, S.~Tu, and
  N.~Matni.
\newblock Learning control barrier functions from expert demonstrations.
\newblock \emph{arXiv:2004.03315}, 2020.

\bibitem[Sindhwani et~al.(2018)Sindhwani, Tu, and
  Khansari-Zadeh]{sindhwani18vectorfields}
V.~Sindhwani, S.~Tu, and S.~M. Khansari-Zadeh.
\newblock Learning contracting vector fields for stable imitation learning.
\newblock \emph{arXiv:1804.04878}, 2018.

\bibitem[{Singh} et~al.(2017){Singh}, {Majumdar}, {Slotine}, and
  {Pavone}]{sumeet_icra}
S.~{Singh}, A.~{Majumdar}, J.~{Slotine}, and M.~{Pavone}.
\newblock Robust online motion planning via contraction theory and convex
  optimization.
\newblock In \emph{2017 IEEE International Conference on Robotics and
  Automation (ICRA)}, pages 5883--5890, 2017.

\bibitem[Singh et~al.(2019)Singh, Richards, Slotine, Sindhwani, and
  Pavone]{singh19learning}
S.~Singh, S.~M. Richards, J.-J.~E. Slotine, V.~Sindhwani, and M.~Pavone.
\newblock Learning stabilizable nonlinear dynamics with contraction-based
  regularization.
\newblock \emph{arXiv:1907.13122}, 2019.

\bibitem[Slotine and Li(1991)]{slot_li_book}
J.-J. Slotine and W.~Li.
\newblock \emph{Applied Nonlinear Control}.
\newblock 1991.

\bibitem[Sontag(1989)]{sontag89universal}
E.~D. Sontag.
\newblock A 'universal' construction of artstein's theorem on nonlinear
  stabilization.
\newblock \emph{Systems \& Control Letters}, 13\penalty0 (2):\penalty0
  117--123, 1989.

\bibitem[Srebro et~al.(2010)Srebro, Sridharan, and Tewari]{srebro10smoothness}
N.~Srebro, K.~Sridharan, and A.~Tewari.
\newblock Smoothness, low-noise and fast rates.
\newblock In \emph{Neural Information Processing Systems}, 2010.

\bibitem[Strogatz(1994)]{strogatz}
S.~H. Strogatz.
\newblock \emph{Nonlinear Dynamics and Chaos}.
\newblock 1994.

\bibitem[Taylor et~al.(2019)Taylor, Singletary, Yue, and Ames]{taylor19cbf}
A.~Taylor, A.~Singletary, Y.~Yue, and A.~Ames.
\newblock Learning for safety-critical control with control barrier functions.
\newblock \emph{arXiv:1912.10099}, 2019.

\bibitem[Vershynin(2019)]{vershynin09gfa}
R.~Vershynin.
\newblock Lectures in geometric functional analysis, 2019.

\bibitem[Wainwright(2019)]{wainwright19book}
M.~J. Wainwright.
\newblock \emph{High-Dimensional Statistics: A Non-Asymptotic Viewpoint}.
\newblock 2019.

\bibitem[Wensing and Slotine(2020)]{wensing2018convexity}
P.~M. Wensing and J.-J.~E. Slotine.
\newblock Beyond convexity -- contraction and global convergence of gradient
  descent.
\newblock \emph{PLoS One}, 15\penalty0 (8):\penalty0 e0236661, 2020.

\end{thebibliography}
}
\appendix

\section{Experiment Details}
\label{sec:app:experiments}
\subsection{Metric learning algorithm}
Here we give pseudocode for the metric learning algorithm used in the main text. Algorithm~\ref{alg:metric} is written for a parameterization $M_{w}(x)$ that yields a convex optimization problem, and where uniform positive definiteness may be enforced globally, such as via SOS matrix constraints. It may be readily relaxed to nonconvex parameterizations such as neural networks by using soft constraints and minimizing the loss using a variant of stochastic gradient descent. Uniform positive definiteness can be imposed along trajectories rather than globally.
\begin{center}
\begin{algorithm}[h!]
\caption{Metric learning}\label{alg:metric}
    \begin{algorithmic}[1]
        \STATE{\bf Hyperparameters:} timestep $\Delta t$,
        set $X \subseteq \R^p$, horizon length $T$, linear approximation tolerance $\epsilon$, number of samples $n$, lower bound $\mu$, overshoot $L$, contraction rate $\lambda$.
        \STATE{\texttt{\# Generate samples.}}
        \WHILE{number of samples less than $n$}
        \STATE{Draw $x^{(1)} \in X$ from a distribution $\calD$ on $X$.}
        \STATE{\texttt{\# Rejection sample to ensure that $x^{(1)} + \delta x \in X$.}}
        \STATE{Draw $\delta x \in \R^p$ uniformly from a ball of radius $\epsilon$ around the origin.}
        \STATE{Set $x^{(2)} := x^{(1)} + \delta x$.}
        \STATE{Compute the flows $\varphi_t(x^{(1)}), \varphi_t(x^{(2)})$ for $t \in [0, T]$ with timestep $\Delta t$.}
        \IF{$\Vert \varphi_t(x^{(1)}) - \varphi_t(x^{(2)})\Vert\ \leq L \epsilon\ \forall\ t\in [0, T]$}
        \STATE{Compute numerical time derivatives of $\varphi_t(x^{(1)})$ and $\psi_t(\delta x) := \varphi_t(x^{(1)}) - \varphi_t(x^{(2)})$.}
        \STATE{Add $\varphi_t(x^{(1)})$, $\psi_t(\delta x)$, $\frac{d}{dt} \varphi_t(x^{(1)})$, $\frac{d}{dt} \psi_t(\delta x)$ to the dataset.}
        \STATE{Increment the number of samples.}
        \ENDIF
        \ENDWHILE
        \STATE{\bf{Solve} the optimization problem:}
        \begin{align*}
            &\min_w \frac{1}{2}\Vert w \Vert^2\\
            &~\mathrm{s.t.}\ \  \frac{d}{dt} \ip{\psi_t(\delta x_i)}{M_w(\varphi_t(x_i)) \psi_t(\delta x_i)} \leq -\lambda \ip{\psi_t(\delta x_i)}{M_w(\varphi_t(x_i)) \psi_t(\delta x_i)} \:, \:\:  i=1, \hdots, n \:,\\
            &~~~~~~~~~M_w(x) \succeq \mu I\ \forall x \in \R^p \:.
        \end{align*}
    \end{algorithmic}
\end{algorithm}
\end{center}

\subsection{Pendulum}

The pendulum dynamics are given by
$m \ell^2 \ddot{\theta} + b \dot{\theta} + m g \ell \sin{\theta} = 0$ with
$m = 1$, $g=9.81$, $\ell=1$, $b=2$.
The state space is $x = (\theta, \dot{\theta})$ and
the stable equilibrium is $x = 0$ with $\theta$ wrapped to the interval $(-\pi, \pi]$.

We generate $n=1000$ trajectories
initialized at $x_0 \sim \mathrm{Unif}([-2, 2] \times [-2,  2])$.
Each trajectory is rolled out
using the default integrator for \texttt{scipy.integrate.solve\_ivp}
for
$T=8$ seconds with $dt = 0.02$, yielding a dataset
of size $1000 \times 400 \times 2$.
We use scipy's \texttt{savgol\_filter}
with $\texttt{window\_length}=5$
and $\texttt{polyorder}=2$ to numerically
compute the derivatives $\dot{x}$.
We set the hidden
width $h=30$ and minimize the loss
$L(\theta) = \sum_{i=1}^{1000} \sum_{k=1}^{400} \mathsf{ReLU}( \ip{\nabla V_\theta(x_i(k))}{\dot{x}_i(k)} + \gamma V_\theta(x_i(k)) )  + \lambda \norm{\theta}^2$,
setting $\gamma=0.01$ and $\lambda=0.1$.
The loss is minimized for $1000$ epochs with Adam using
a step size $10^{-3}$
and a batch size of $1000$.

We repeat this experiment for $30$ trials.
For each trial we use a test set of size $1000$
to compute a UCB on the generalization error.
The 10/50/90-th
percentile of the UCBs are
$0.459\%$, $0.459\%$, and $1.163\%$.

We also uniformly grid the set $[-2, 2] \times [-4, 4]$
with $40000$ points
and check numerically how often the condition $\ip{\nabla V_\theta(x)}{\dot{x}} \leq - \gamma V_\theta(x)$
is violated.
The 10/50/90-th
percentile of these violations over $30$ trials are
$0.168\%$, $1.548\%$, and $3.696\%$.
We note that these numbers are higher than the
generalization error because the set
$[-2, 2] \times [-4, 4]$ contains points that
are outside the flow starting
from $[-2, 2] \times [-2, 2]$.

For our adaptive control experiments, the
dynamics combined with the added disturbance are
\begin{align*}
    m \ell^2 \ddot{\theta} + b \dot{\theta} + m g \ell \sin{\theta} + \ip{a}{\kappa \phi(t)} = u \:,
\end{align*}
where $u \in \R$ is the control input.
We sample $a \in \R^{10}$ from $a \sim N(0, I)$.
We set $\phi(t) = (\sin(\omega_1 t), ..., \sin(\omega_{10} t))$
where each $\omega_i \stackrel{\mathrm{i.i.d}}{\sim} \mathrm{Unif}([0, 2\pi])$
The adaptive control law we use is
$u(t) = \ip{\hat{a}(t)}{\kappa \phi(t)}$
where $\hat{a}(t)$ evolves according to:
\begin{align}
    \dot{\hat{a}}(t) = - \gamma \phi(t) \ip{\nabla_x V_\theta(x(t))}{e_2} \:, \:\: \hat{a}(0) = 0 \:, \label{eq:pendulum_adaptive}
\end{align}
where $\gamma = 15$,
$V_\theta$ is the learned Lyapunov function,
and $e_2 = \cvectwo{0}{1} \in \R^2$.
The idea behind the adaptive control law \eqref{eq:pendulum_adaptive}
is to rely on the nominal stability of the
pendulum dynamics and learn to cancel out the
uncertain disturbance, similar to the adaptive law presented by~\citet{brett_adapt}.
We note that this adaptive law is also a special case of a more general class
of speed-gradient algorithms from \citet{fradkov99}.
We give a self-contained proof of its correctness.
\begin{lem}
Consider the dynamical system
\begin{align}
    \dot{x} = f(x) + B( u(t) - Y(x, t) a ) \:, \label{eq:matched_uncertainty}
\end{align}
with $f$ continuously differentiable and $Y(x,t)$ locally bounded in $x$ uniformly in $t$.
Let $V$ be a twice continuously differentiable positive definite function such that $\ip{\nabla V(x)}{f(x)} \leq -\rho(x)$ for all $x$ for some continuously differentiable positive definite function $\rho(\cdot)$.
Let $M$ be a positive definite matrix.
Let $\hat{a}(t)$ be defined by the differential equation
\begin{align}
    \dot{\hat{a}} = - M^{-1} Y(x, t)^\T B^\T \nabla V(x(t)) \:.
\end{align}
Then the adaptive control law
\begin{align}
    u(t) = Y(x, t) \hat{a}(t) \:,
\end{align}
in feedback with \eqref{eq:matched_uncertainty}
drives $x \to 0$ and $\dot{x} \to 0$.
\end{lem}
\begin{proof}
Let $\ip{x}{y}_M = x^\T M y$
and $\norm{x}_M^2 = \ip{x}{x}_M$.
Define the new candidate Lyapunov function
\begin{align}
    \bar{V}(t) := V(t) + \frac{1}{2} \norm{\hat{a} - a}_M^2 \:.
\end{align}
Let $\tilde{a} := \hat{a} - a$.
Differentiating $\bar{V}$ with respect to time:
\begin{align*}
    \dot{\bar{V}} &= \ip{\nabla V(x)}{f(x) + B(u - Y a)} + \ip{\tilde{a}}{\dot{\hat{a}}}_M \\
    &= \ip{\nabla V(x)}{f(x)} + \ip{\nabla V(x)}{B Y \tilde{a}} + \ip{\tilde{a}}{\dot{\hat{a}}}_M \\
    &= \ip{\nabla V(x)}{f(x)} + \ip{Y^\T B^\T \nabla V(x)}{ \tilde{a}} + \ip{\tilde{a}}{\dot{\hat{a}}}_M \\
    &= \ip{\nabla V(x)}{f(x)} + \ip{M^{-1} Y^\T B^\T \nabla V(x)}{ \tilde{a}}_M + \ip{\tilde{a}}{\dot{\hat{a}}}_M \\
    &= \ip{\nabla V(x)}{f(x)}\\
    &\leq -\rho(x) \:.
\end{align*}
Since $-\rho(x) < 0$ for all $x \neq 0$,
this shows that $\bar{V}$ is bounded for all $t$,
which implies both that $V$ is bounded and that
$\tilde{a}$ is bounded for all $t$.
Since $V$ is positive definite, $V$ bounded implies that
$x$ is bounded. Integrating the above inequality shows that
\begin{equation*}
    \int_0^\infty \rho(x(\tau))d\tau \leq \bar{V}(0) \:,
\end{equation*}
so that $\rho \in \mathcal{L}_1$. Now, $\dot{\rho} = \ip{\nabla \rho}{f(x) + BY\tilde{a}}$. By continuity of $\nabla \rho$, $f$, and $Y$, and by boundedness of $\tilde{a}$, $\dot{\rho}$ is bounded. Hence $\rho$ is uniformly continuous, and by Barbalat's Lemma (see e.g. Lemma 4.2 of~\citet{slot_li_book}) $\rho \rightarrow 0$. By positive definiteness of $\rho$, $\rho \rightarrow 0$ implies that $x\rightarrow 0$ and $\dot{x} \rightarrow 0$.
\end{proof}

\subsection{Minitaur}

We collect $50000$ random training trajectories
and $10000$ random test trajectories
using the same distribution over the impulse force.
For each trajectory, we step the simulator $200$
times at $dt = 0.002$. The state dimension
excluding the base orientation is $16$.

The PD controller
is able to return the joint angles and velocities (excluding the base orientation)
to their original standing position up to a small
bias of size $\sim 0.2$ in $\ell_2$-norm.
Therefore, we train a discrete-time
Lyapunov function $V_\theta$ to satisfy
$V_\theta(e_i(k+1)) \leq \rho V_\theta(e_i(k)) + \gamma$ where $e_i(k)$ is the error state of the $i$-th trajectory at the $k$-th step.
The specific values we use are $(\rho, \gamma) = (0.945, 0.025)$.
The extra slack term $\gamma$ is necessary
for the Lyapunov function to converge to a ball
instead of zero.

We use
a hidden width of $h=40$ and minimize the loss
$L_n(\theta) = \sum_{i=1}^{n} \sum_{k=1}^{199} \mathsf{ReLU}( V_\theta(e_i(k+1)) - \rho V_\theta(e_i(k)) - \gamma) + \lambda \norm{\theta}^2$,
setting $\lambda = 0.01$.
The loss is minimized for $1000$ epochs with Adam using
a step size $10^{-3}$ with cosine decay\footnote{See \url{https://www.tensorflow.org/api_docs/python/tf/compat/v1/train/cosine_decay}.}
and a batch size of $1000$.

\subsection{6-dimensional gradient system}
\label{app:grad}
We parametrize the metric via monomials up to degree two in the state variables. We enforce global positive definiteness $M(x) \geq \mu I$ via SOS matrix constraints and set $\mu = 1$. We use a tolerance of $5\times 10^{-3}$ for the size of each perturbation $\delta x$. A pair of trajectories $\varphi_t(x_1)$, $\varphi_t(x_2)$ with $x_2 = x_1 + \delta x$ is considered to generate a trajectory $\delta x(t) = \varphi_t(x_2) - \varphi_t(x_1)$ if $\Vert\delta x(t)\Vert < 10^{-2}$ for all $t$, so that a small overshoot is permitted. Pairs of trajectories not satisfying this requirement are discarded until the desired number of training samples is reached. The size of this overshoot parameter sets the maximum allowed $L$ where $M(x) \leq L$ for any metric learned, as we impose $M(x) \geq I$. In general, the overshoot with respect to the Euclidean norm is given by $\sqrt{\frac{L}{l}}$ for $l I \leq M(x) \leq L I$. In practice, we require the maximum bound on $\Vert \delta x(t)\Vert$ to be sufficiently small that the dynamics of $\varphi_t(x_1) - \varphi_t(x_2)$ well-approximates the variational system on the trajectory $\varphi_t(x_1)$. To search for metrics with larger values of $\frac{L}{l}$, we can vary $\mu$, $\Vert \delta x\Vert$, and the maximum allowed $\Vert\delta x(t)\Vert$ while ensuring $\Vert\delta x(t)\Vert$ remains small throughout its entire trajectory.

Each trajectory is simulated until $T=2$ seconds with a timestep $dt = 5\times 10^{-3}$. Because we are interested in convergence of the variational dynamics, we use a small time horizon. This generates a dataset of size $n \times 400 \times 12$ where $n$ is the number of trajectories and $12$ is the dimension of the tangent bundle. We subsequently downsample and impose $25$ differential Lyapunov constraints along each trajectory. We search for a metric with a rate $\lambda = 4$. Initial conditions are drawn uniformly from the ball of radius $r=3$.

The time derivatives $\dot{x}$ and $\delta\dot{x}$ are computed numerically by fitting a cubic spline to the corresponding trajectories and analytically differentiating the spline. $M(x)$ is found by minimizing $\Vert w\Vert^2$ where $w$ is a vector containing all parameters. The test set is of size $1000$ and each data point in Figure~\ref{fig:lyap_exps} (LR) was computed by averaging over $25$ independent draws of the training set. The RCP bound is obtained with a confidence of $\delta = .01$.

\subsubsection{Randomized convex programs}

Consider the following optimization problem:
\begin{align}
    \min_{x \in X} \ip{c}{x} : f(x, \theta) \leq 0 \:\:  \forall \theta \in \Theta \:. \label{eq:rcp_problem}
\end{align}
We assume that $X \subseteq \R^d$ is a convex set and
$x \mapsto f(x, \theta)$ is convex for every $\theta \in \Theta$.
In the case where $\Theta$ is an infinite (or very large) set, we consider approximations to
\eqref{eq:rcp_problem} formulated as follows.
Let $\nu$ denote a distribution over $\Theta$.
Let $\theta_1, ..., \theta_n$ be i.i.d.\ samples
from $\nu$.
Let $\hat{x}_n$ denote a solution to:
\begin{align}
    \min_{x \in X} \ip{c}{x} : f(x, \theta_i) \leq 0 \:, \:\: i = 1, ..., n \:.
\end{align}

\begin{thm}[See e.g. Theorem 3.1 of \citet{calafiore10rcp}]
\label{thm:rcp}
Fix any $\varepsilon \in [0, 1]$.
Define $\beta(\varepsilon)$ as:
\begin{align*}
    \beta(\varepsilon) := \sum_{i=0}^{d-1} { n \choose i } \varepsilon^{i} (1 - \varepsilon)^{n-i} \:.
\end{align*}
With probability at least $1 - \beta(\varepsilon)$
over $\theta_1, ..., \theta_n$,
we have that:
\begin{align*}
    \Pr_{\theta \sim \nu}( f(\hat{x}_n, \theta) > 0) \leq \varepsilon \:.
\end{align*}
\end{thm}

We use Theorem~\ref{thm:rcp} as follows.
We fix a failure probability $\delta \in (0, 1)$.
We then numerically solve for $\varepsilon$
such that $\beta(\varepsilon) = \delta$.

To understand the scaling of $\varepsilon$ as a function
of $n$ and $\delta$,
despite the lack of closed form expression, consider the following.
If $n \geq d$ then
by a Chernoff bound on the CDF of a Binomial random variable~(c.f.\ Section 5 of \citet{calafiore10rcp}), we can derive the upper bound
\begin{align*}
    \varepsilon \leq c \frac{(d - 1 + \log(1/\delta))}{n} \:,
\end{align*}
where $c$ is a universal constant.

\subsection{Van der Pol}
\label{app:vdp}
The study of the Van der Pol (VDP) oscillator was foundational to the development of nonlinear dynamics~\citep{strogatz}, and its global contraction properties have been analyzed algoithmically via SOS programming~\citep{sos_contr}. We study the damped VDP to visualize the violation set for the metric condition categorized in Theorem~\ref{thm:global_diff_lyap}. The dynamics of the damped Van der Pol are given by
\begin{equation*}
    \ddot{x} + \alpha(x^2 + k)\dot{x} + \omega^2x = 0 
\end{equation*}
We set $\alpha = k = \omega = 1$.

We parameterize the metric via monomials up to degree four in the state variables. Similar to the gradient system, we set $M(x) \geq I$ and use a tolerance $\Vert \delta x\Vert \leq 5\times 10^{-3}$, $\Vert \delta x(t) \Vert \leq 10^{-2}$ for all $t \in [0, T]$. We use a timestep $dt = 5\times 10^{-3}$ and simulate until a final time $T = 3$ seconds. This generates a dateset of size $n \times 600 \times 4$ where $4$ is the dimension of the tangent bundle and $n$ is the number of training trajectories.

We subsequently downsample and impose $50$ differential Lyapunov constraints along each trajectory with a rate $\lambda = 3/4$. Initial conditions are drawn uniformly from a ball of radius $r = 2$. The same techniques are used for numerical differentiation as for the gradient system.

Theorem~\ref{thm:global_diff_lyap} predicts that the size of the violation set will decrease as the rate tested for the metric condition $\eta\lambda$ with $0 < \eta < 1$ decreases, or as the number of training samples $n$ increases. In both Figure~\ref{fig:vdp_pp} and Figure~\ref{fig:vdp_pp_2}, we use a uniform grid with $9\times 10^4$ points over $[-2, 2]\times [-2, 2]$ to test the metric condition $\frac{\p f}{\p x}(x)^\T M(x) + M(x)\frac{\p f}{\p x}(x) + \dot{M}(x) \leq -2\eta\lambda$ for the learned metric and the true dynamics.

In Figure~\ref{fig:vdp_pp}, we plot the violation set for fixed $\eta = 1$ as a function of the number of training samples. As discussed in the main text, the size of the violation set decreases, and it is pushed to the boundary of the sampled region as $n$ increases.

In Figure~\ref{fig:vdp_pp_2}, we plot the violation set as a function of $\eta$. As $\eta$ decreases to zero, the size of the violation set decreases and is pushed to the boundary of the sampled region.

\begin{figure}[ht]
    \centering
    \includegraphics[width=.45\textwidth]{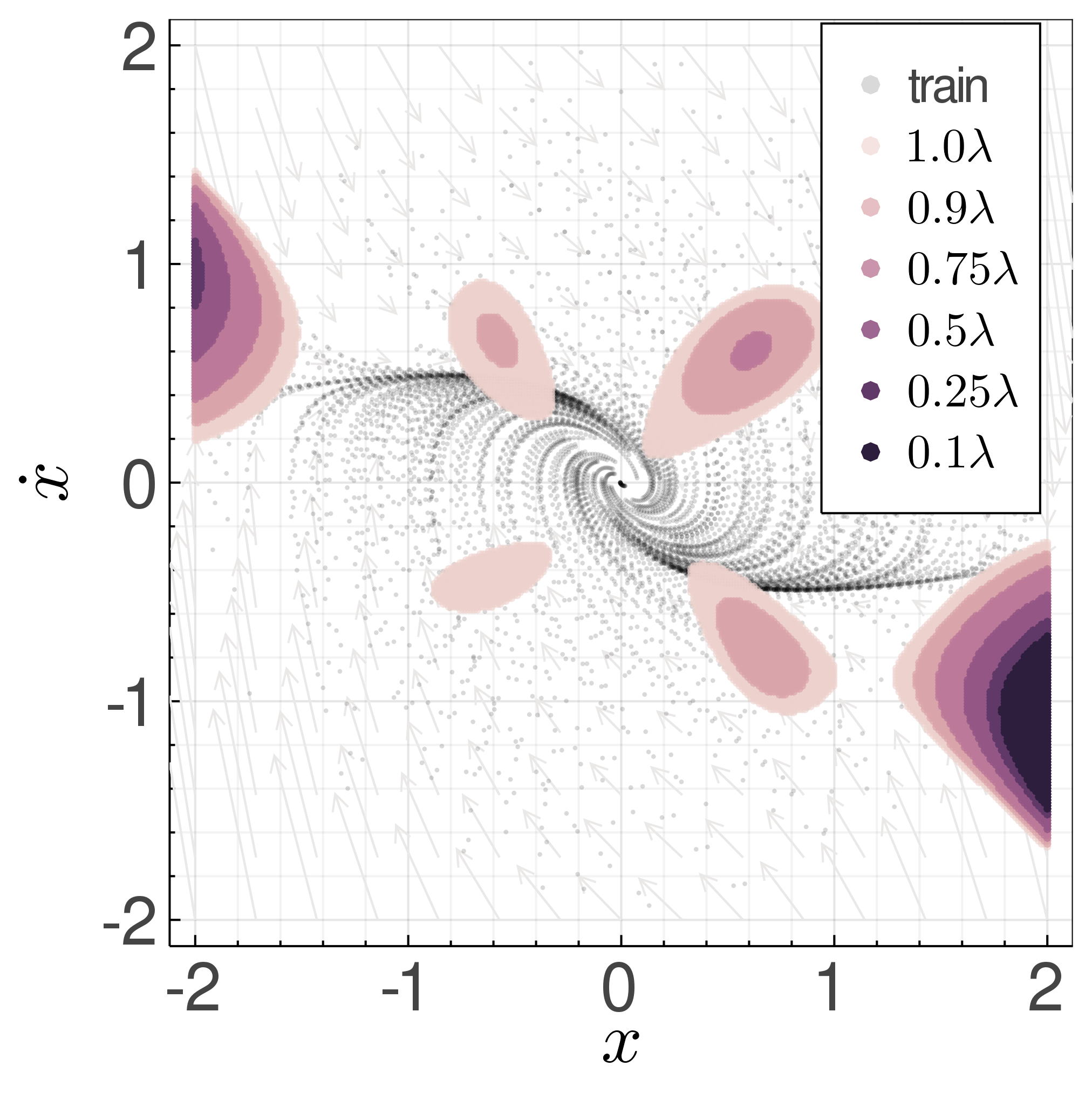}
    \caption{Violation set for contraction metric learning as a function of the tested contraction rate $\eta \lambda$ with $0 < \eta < 1$ for the damped VDP system.}
    \label{fig:vdp_pp_2}
\end{figure}

\section{Proofs for Section~\ref{sec:local:bounds}}
\label{sec:app:local}

Recall that Dudley's inequality gives us the
following estimate on $\calR_n(\calV)$:
\begin{align}
    \calR_n(\calV) \leq \frac{24 L_h}{\sqrt{n}} \int_0^\infty \sqrt{\log N(\varepsilon; \calV, \norm{\cdot}_{\calV})} \: d\varepsilon \:. \label{eq:dudley}
\end{align}

\subsection{Proof of Theorem~\ref{thm:gen_bound_linear}}

For every $V_{\theta_1}, V_{\theta_2} \in \calV$,
$\norm{ V_{\theta_1} - V_{\theta_2} }_{\calV} \leq (L_g + L_{\nabla g}) \norm{\theta_1 - \theta_2}$.
Furthermore, a standard volume comparison
argument tells us that
$\log{N(\varepsilon; \mathbb{B}_2^k(1), \norm{\cdot})} \leq k \log(1 + 2/\varepsilon)$,
where $\mathbb{B}_2^k(1)$ is the closed $\ell_2$-ball in $\R^k$ of radius $1$.
Therefore, by \eqref{eq:dudley}:
\begin{align*}
    \calR_n(\calV) &\leq \frac{24 B_\theta (L_g + L_{\nabla g})L_h }{\sqrt{n}} \int_0^\infty \sqrt{\log N(\varepsilon; \mathbb{B}_2^k(1), \norm{\cdot})} \: d\varepsilon \leq O(1) B_\theta (L_g + L_{\nabla g}) L_h \sqrt{\frac{k}{n}} \:.
\end{align*}
Theorem~\ref{thm:gen_bound_linear} now follows
from Lemma~\ref{lemma:fast_rate}.

\subsection{Proof of Theorem~\ref{thm:gen_bound_sos}}

A simple calculation shows that
$\norm{V_{Q_1} - V_{Q_2}}_{\calV} \leq (B_m^2 + 2 B_{Dm} B_m) \norm{Q_1 - Q_2}$.
Hence by \eqref{eq:dudley},
\begin{align*}
    \calR_n(\calV) \leq \frac{24B_Q (B_m^2 + 2 B_{Dm} B_m) L_h}{\sqrt{n}} \int_0^\infty \sqrt{ \log N(\varepsilon; \mathbb{B}_2^{d \times d}(1), \norm{\cdot}) } \: d\varepsilon \:.
\end{align*}
Here, $\mathbb{B}_2^{d \times d}(1)$ is the closed ball of $d \times d$ matrices
with Frobenius norm bounded by $1$, and $\norm{\cdot}$ for matrices
denotes the operator norm.
The metric entropy $\log{N(\varepsilon;\mathbb{B}_2^{d \times d}(1), \norm{\cdot})}$ can
be bounded by the minimum of the standard
volume comparison bound and applying the dual Sudakov inequality (see e.g. Theorem 2.2 of \citet{vershynin09gfa}):
\begin{align*}
    \log{N(\varepsilon;\mathbb{B}_2^{d \times d}(1), \norm{\cdot})} \leq \min\{ d^2 \log(1+2/\varepsilon), c d/\varepsilon^2 \} \:.
\end{align*}
Here, $c$ is an absolute constant.
By integrating this estimate we arrive at the bound:
\begin{align}
    \calR_n(\calV) \leq O(1) B_Q (B_m^2 + B_{D_m} B_m)L_h  \sqrt{\frac{d}{n}} \log{d} \:.
\end{align}
Theorem~\ref{thm:gen_bound_sos} now follows
from Lemma~\ref{lemma:fast_rate}.

\subsection{Proof of Theorem~\ref{thm:gen_bound_rkhs}}

Define the function classes $\calF(B)$ and $\hat{\calF}(B, \{\theta_k\})$ as:
\begin{align}
    \calF(B) &:= \left\{ f(x) = \int_{\Theta} \alpha(\theta) \phi(x; \theta) \: d\theta : \norm{f}_\nu := \sup_{\theta \in \Theta} \bigabs{\frac{\alpha(\theta)}{\nu(\theta)}} \leq B \right\} \:, \\
    \hat{\calF}(B, \{\theta_k\}_{k=1}^{K}) &:= \left\{ f(x) = \sum_{k=1}^{K} c_k \phi(x; \theta_k) : \norm{c}_1 \leq B \right\} \:.
\end{align}

The following is a simple modification
of Theorem 3.2 from \citet{rahimi08uniform},
which also accounts for uniform approximation
of the derivatives.
\begin{lem}
\label{lem:function_finite_approx}
Fix a $B > 0$ and a bounded space $X \subseteq \R^p$.
Let $B_X := \sup_{x \in X} \norm{x}$.
Suppose that $\Theta\subseteq \R^p \times \R$
with $B_\theta := \sup_{\theta \in \Theta} \norm{\theta}$.
Furthermore, suppose that
$\phi(x; \theta) = \phi(\ip{x}{w} + b)$,
$\phi$ is $L_\phi$-Lipschitz, and $\abs{\phi} \leq 1$.
Fix a $f \in \calF(B)$ and $\delta \in (0, 1)$.
Let $\theta_1, ..., \theta_K$ be i.i.d. draws from $\nu$.
With probability at least $1-\delta$,
there exists a $\hat{f} \in \hat{\calF}(\norm{f}_\nu, \{\theta_k\}_{k=1}^{K})$ such that:
\begin{align}
    \sup_{x \in X} \abs{\hat{f}(x) - f(x)} \leq \frac{2\norm{f}_\nu}{\sqrt{K}} \left( 1 + \sqrt{\log(1/\delta)} + 2 L_\phi\left(B_X \sqrt{\E_\nu \norm{w}^2} + \sqrt{\E_\nu \abs{b}^2} \right) \right) \:. \label{eq:approx_f}
\end{align}
Furthermore, now assume that $\phi$ is differentiable
and $\phi'$ is $L_{\phi'}$-Lipschitz.
Then every $f \in \calF(B)$ is differentiable with
\begin{align}
\nabla f(x) = \int_\Theta \alpha(\theta) \nabla \phi(x; \theta) \: d\theta \: \label{eq:grad_f_expr} \:.
\end{align}
Finally, with probability at least $1-2\delta$,
there exists a $\hat{f} \in \hat{\calF}(\norm{f}_\nu, \{\theta_k\}_{k=1}^{K})$
such that \eqref{eq:approx_f} holds
and also:
\begin{align}
    \sup_{x \in X} \norm{\nabla \hat{f}(x) - \nabla f(x)} \leq \frac{B_\theta \norm{f}_\nu}{\sqrt{K}} \left( L_\phi \sqrt{2 \log(1/\delta)} + 4 (L_\phi + B_\theta L_{\phi'})(B_X + 1) \sqrt{p} \right) \:. \label{eq:approx_grad_f}
\end{align}
\end{lem}
\begin{proof}
Following the proof of Theorem 3.2
of \citet{rahimi08uniform},
we set $c_k = \frac{\alpha(\theta_k)}{K \nu(\theta_k)}$,
and we define
$v(\theta_1, ..., \theta_K) = \norm{ \hat{f} - f }_\infty$.
It is shown in \cite{rahimi08uniform} that
for all $\theta_1, ..., \theta_K, \theta'_k \in \Theta$:
\begin{align*}
    \abs{v(\theta_1, ..., \theta_k, ..., \theta_K) - v(\theta_1, ..., \theta'_k, ..., \theta_K)} \leq \frac{2\norm{f}_\nu}{K} \:.
\end{align*}
Next, we control the expected value of $\nu$:
\begin{align*}
    \E v(\theta_1, ..., \theta_k) &= \E \sup_{x \in X} \abs{\hat{f}(x) - \E \hat{f}(x) } \\
    &\leq 2 \E_\theta \E_\varepsilon \sup_{x \in X} \bigabs{\sum_{k=1}^{K} \varepsilon_k c_k \phi(\ip{x}{w_k} + b_k)} \\
    &= 2 \E_\theta \E_\varepsilon \sup_{x \in X} \bigabs{\sum_{k=1}^{K} \varepsilon_k c_k ( \phi(\ip{x}{w_k} + b_k) - \phi(0)) + \sum_{k=1}^{K} \varepsilon_k c_k \phi(0) } \\
    &\leq 2 \E_\theta\E_\varepsilon \sup_{x \in X} \bigabs{ \sum_{k=1}^{K} \varepsilon_k c_k ( \phi(\ip{x}{w_k} + b_k) - \phi(0)) } + 2 \E_\theta \E_\varepsilon \bigabs{ \sum_{k=1}^{K} \varepsilon_k c_k \phi(0) } \\
    &=: T_1 + T_2 \:.
\end{align*}
Above, the first inequality is a standard
symmetrization argument where the Rademacher
variables $\{\varepsilon_k\}$ are introduced,
and the second inequality is the triangle inequality.
We first bound $T_1$.
Set $\psi_k(z) := c_k( \phi(z) - \phi(0))$.
Clearly $\psi_k(0) = 0$, and
also $\psi_k$ is $c_k L_\phi$-Lipschitz.
We bound $c_k L_\phi \leq \frac{\norm{f}_\nu L}{K}$.
Therefore by
the contraction inequality for Rademacher
complexities~(Theorem 4.12 of \citet{ledoux91book}) followed by Jensen's inequality:
\begin{align*}
    T_1 &\leq \frac{4 \norm{f}_\nu L}{K} \E_\theta\E_\varepsilon \sup_{x \in X} \bigabs{\sum_{k=1}^{K} \varepsilon_k (\ip{x}{w_k} + b_k)} \\
    &\leq \frac{4\norm{f}_\nu L B_X}{K} \E_\theta\E_\varepsilon \bignorm{ \sum_{k=1}^{K} \varepsilon_k w_k }  + \frac{4\norm{f}_\nu L}{K} \E_\theta \E_\varepsilon \bigabs{\sum_{k=1}^{K} \varepsilon_k b_k } \\
    &\leq \frac{4 \norm{f}_\nu L (B_X \sqrt{\E \norm{w_1}^2} + \sqrt{\E \abs{b_1}^2})}{\sqrt{K}} \:.
\end{align*}
Furthermore we can bound
$T_2 \leq \frac{2 \norm{f}_\nu}{\sqrt{K}}$ by similar
arguments.
The claim \eqref{eq:approx_f} now follows by invoking McDiarmid's inequality.

We note that \eqref{eq:grad_f_expr}
follows from a basic application
of the dominated convergence theorem,
since we have that:
\begin{align*}
    \E_{\theta \sim \nu} \sup_{x \in X} \norm{\nabla \phi(x; \theta)} < \infty \:.
\end{align*}

Finally, we focus on the derivative condition \eqref{eq:approx_grad_f}.
Let $g(\theta_1, ..., \theta_K) := \sup_{x \in X} \norm{\nabla \hat{f}(x) - \nabla f(x)}$.
By symmetrization we have:
\begin{align*}
    \E g(\theta_1, ..., \theta_K) &= \E \sup_{x \in X} \norm{ \nabla \hat{f}(x) - \E \nabla{f}(x) } \\
    &= \E \sup_{x \in X} \sup_{\norm{q} = 1} \ip{q}{\nabla \hat{f}(x)} - \E \ip{q}{\nabla f(x)} \\
    &\leq 2 \E_\theta \E_\varepsilon \sup_{x \in X} \sup_{\norm{q} = 1} \sum_{k=1}^{K} \varepsilon_k c_k \phi'(\ip{x}{w_k} + b_k) \ip{q}{w_k} \:.
\end{align*}
We set $\psi_k(x, q)$ to be:
\begin{align*}
    \psi_k(x, q) := c_k \phi'(\ip{x}{w_k} + b_k) \ip{q}{w_k}
\end{align*}
For $(x_1, q_1), (x_2, q_2) \in X \times \mathbb{B}_2^p(1)$ we have:
\begin{align*}
    \abs{\psi_k(x_1, q_1) - \psi_k(x_2, q_2)} &= c_k \abs{ \phi'(\ip{x_1}{w_k} + b_k) \ip{q_1}{w_k} - \phi'(\ip{x_2}{w_k} + b_k) \ip{q_2}{w_k} } \\
    &\leq c_k (L_\phi \abs{\ip{q_1 - q_2}{w_k}} + \abs{\phi'(\ip{x_1}{w_k} + b_k) - \phi'(\ip{x_2}{w_k} + b_k)}\abs{\ip{q_2}{w_k}}) \\
    &\leq c_k (B_\theta L_\phi \norm{q_1 - q_2} + L_{\phi'} B_\theta^2 \norm{x_1 - x_2}) \\
    &\leq \frac{\norm{f}_\nu}{K} B_\theta (L_\phi + B_\theta L_{\phi'}) \sqrt{2} \bignorm{\cvectwo{x_1}{q_1} - \cvectwo{x_2}{q_2}} \:.
\end{align*}
We can now apply Theorem 3 of \citet{maurer16vectorcontraction} to conclude that:
\begin{align*}
     &2 \E_\theta \E_\varepsilon \sup_{x \in X} \sup_{\norm{q} = 1} \sum_{k=1}^{K} \varepsilon_k c_k \phi'(\ip{x}{w_k} + b_k) \ip{q}{w_k} \\
     &= 2\E_\theta \E_\varepsilon \sup_{x \in X} \sup_{\norm{q} = 1} \sum_{k=1}^{K}\varepsilon_k \psi_k(x, q) \\
     &\leq 4 \frac{\norm{f}_\nu}{K} B_\theta (L_\phi + B_\theta L_{\phi'}) \E_\varepsilon \sup_{x \in X} \sup_{\norm{q} = 1} \sum_{k=1}^{K} \bigip{\varepsilon_k}{ \cvectwo{x}{q} } \\
     &\leq 4 \frac{\norm{f}_\nu}{\sqrt{K}} B_\theta (L_\phi + B_\theta L_{\phi'}) (B_X + 1) \sqrt{p} \:.
\end{align*}
Next, we have for all $\theta_1, ..., \theta_K, \theta'_k \in \Theta$:
\begin{align*}
    &\abs{g(\theta_1, ..., \theta_k, ..., \theta_K) - g(\theta_1, ..., \theta'_k, ..., \theta_K)} \\
    &\leq \sup_{x \in X} \bignorm{ \frac{\alpha(\theta_k)}{K \nu(\theta_k)} \phi'(\ip{x}{w_k} + b_k) w_k -  \frac{\alpha(\theta'_k)}{K \nu(\theta'_k)} \phi'(\ip{x}{w'_k} + b'_k) w'_k} \\
    &\leq \frac{2 B_\theta L_\phi \norm{f}_\nu}{K} \:.
\end{align*}
The uniform bound on the derivatives \eqref{eq:approx_grad_f} now follows
from another application of McDiarmid's inequality.
\end{proof}

\newcommand{\calEapprox}{\calE_{\mathrm{approx}}}

We now turn to the proof of Theorem~\ref{thm:gen_bound_rkhs}.
Under the hypothesis of
Lemma~\ref{lem:function_finite_approx},
we have that for every $f \in \calF(B)$:
\begin{align*}
    \sup_{x \in X} \abs{f(x)} &\leq B \:, \\
    \sup_{x \in X} \norm{\nabla f(x)} &\leq B B_\theta L_\phi\:.
\end{align*}
We also have for any $\{\theta_k\}_{k=1}^{K} \subseteq \Theta$ and any $\hat{f}(x) = \sum_{k=1}^{K} c_k \phi(x;\theta_k)$ with $\norm{c}_1 \leq B$,
\begin{align*}
    \sup_{x \in X} \abs{\hat{f}(x)} &\leq B \:, \\
    \sup_{x \in X} \norm{\nabla \hat{f}(x)} &\leq B B_\theta L_\phi \:.
\end{align*}
Hence for any $\{\theta_k\}_{k=1}^{K} \subseteq \Theta$,
the function class
$\hat{\calF}(B_\alpha, \{\theta_k\}_{k=1}^{K})$
satisfies Assumption~\ref{assume:regularity} with
$B_V = B_\alpha$ and $B_{\nabla V} \leq B_\alpha B_\theta L_\phi$.

Let $f_n \in \calF(B_\alpha)$ denote
a feasible solution to \eqref{eq:V_opt}.

At this point, it may be tempting
to use the probabilistic method in conjunction with
Lemma~\ref{lem:function_finite_approx}
to conclude that there exists a set of
weights $\{\overline{\theta}_k\}_{k=1}^{K}$
such that there exists
a $\hat{f}_n \in \hat{\calF}(B_\alpha, \{\overline{\theta}_k\}_{k=1}^{K})$ such that $(\hat{f}_n, \nabla \hat{f}_n)$ closely approximates
$(f_n, \nabla f_n)$.
This will not work however, since the
function class $\hat{\calF}(B_\alpha, \{\overline{\theta}_k\}_{k=1}^{K})$
then becomes a function of the training data $\xi_1, ..., \xi_n$,
and hence we would not be able to apply
Lemma~\ref{lemma:fast_rate} to it.

To work around this, we need to draw the weights
independently of $\xi_1, ..., \xi_n$.
In particular,
we set $K$ such that
\begin{align*}
    K = \bigceil{\frac{ c B_\alpha^2 L_h^2}{\gamma^2} ( (1 + B_\theta L_\phi) \sqrt{\log{n}} + B_\theta (B_S + 1) (L_\phi + B_\theta L_{\phi'}) \sqrt{p} )^2}  \:,
\end{align*}
where $c$ is an absolute constant
and let $\{\theta^*_k\}_{k=1}^{K}$ be drawn i.i.d.\ from $\nu$.

By invoking Lemma~\ref{lem:function_finite_approx}
with $K$ as above and $\delta=1/n^2$,
we know there exists an event $\calEapprox$
on $\nu^{\otimes K}$
such that on $\calEapprox$,
there exists a function $\hat{f}_n \in \hat{\calF}(B_\alpha, \{\theta^*_k\}_{k=1}^{K})$ that satisfies:
\begin{align*}
    \sup_{x \in S} \abs{ f_n(x) - \hat{f}_n(x) } &\leq \gamma/(8\sqrt{2} L_h) \:, \\
    \sup_{x \in S} \norm{\nabla f_n(x) - \nabla \hat{f}_n(x)} &\leq \gamma/(8\sqrt{2}L_h) \:.
\end{align*}
By the definition of $L_h$, these two inequalities imply that
\begin{align*}
    \sup_{\xi \in X} \abs{h(\xi; f_n) - h(\xi; \hat{f}_n)} \leq \gamma / 4 \:.
\end{align*}
This means that
if $f_n$ is feasible for \eqref{eq:V_opt} with slack variable $\gamma$, then on $\calEapprox$ we have that $\hat{f}_n$ is feasible with slack variable $3\gamma/4$.
Specifically:
\begin{align}
    h(\xi_i, \hat{f}_n) \leq -3\gamma/4 \:, \:\: i=1, ..., n \:. \label{eq:hatf_n_feasibility}
\end{align}
Observe then that:
\begin{align*}
    \Pr( h(\xi; f_n) > 0 )
    &\leq \Pr( \{ h(\xi; f_n) > 0) \} \cap \calEapprox ) + \Pr( \calEapprox^c ) \\
    &\leq \Pr( h(\xi, \hat{f}_n) > -\gamma/4) + 1/n^2 \\
    &= \Pr( h(\xi, \hat{f_n}) + \gamma/4 > 0 ) + 1/n^2 \:.
\end{align*}
Here, $\Pr(\cdot)$ denotes the product measure
$\calD \otimes \nu^{\otimes n}$
over $(\xi, \{\theta_k^*\}_{k=1}^{K})$.
Now we define $\tilde{h} = h + \gamma / 4$.
From \eqref{eq:hatf_n_feasibility},
\begin{align*}
    \tilde{h}(\xi_i, \hat{f}_n) = h(\xi, \hat{f}_n) + \gamma/4 \leq -\gamma/2 \:.
\end{align*}
We can then apply Lemma~\ref{lemma:fast_rate}
with $\tilde{h}$ (with the change $B_h \gets B_h + \gamma/4$ and $\gamma \gets \gamma/2$),
to the finite dimensional parametric function class
$\hat{\calF}(B_\alpha, \{\theta^*_k\}_{k=1}^{K})$
(as noted above, this is valid because
the elements $\{\theta^*_k\}_{k=1}^{K}$ are
drawn independently from the training data $\xi_1, ..., \xi_n$).

The result is that with probability at least $1-\delta$ over $\xi_1, ..., \xi_n$:
\begin{align*}
    \Pr( h(\xi, \hat{f}_n) > -\gamma/4 )
    &= \Pr(\tilde{h}(\xi, \hat{f}_n) > 0 ) \\
    &\leq O(1)\left( \frac{\log^3{n}}{\gamma^2} \calR_n^2(\hat{\calF}(B_\alpha, \{\theta^*_k\}_{k=1}^{K})) + \frac{\log(\log(1+B_h/\gamma)/\delta)}{n} \right) \:.
\end{align*}
Letting
$\hat{f} = \sum_{k=1}^{K} c_k \phi(x; \theta^*_k)$
and $\hat{g} = \sum_{k=1}^{K} d_k \phi(x; \theta^*_k)$,
we have that
$\norm{\hat{f} - \hat{g}}_{\calV} \leq (1 + B_\theta L_\phi) \norm{c - d}_1$.
Hence by \eqref{eq:dudley},
\begin{align*}
        \calR_n(\hat{\calF}(B_\alpha, \{\theta^*_k\}_{k=1}^{K})) &\leq \frac{24 B_\alpha(1 + B_\theta L_\phi) L_h}{\sqrt{n}} \int_0^\infty \sqrt{\log N(\varepsilon; \mathbb{B}_1^K(1), \norm{\cdot}_1)} \: d\varepsilon \\
        &\leq O(1) B_\alpha (1 + B_\theta L_\phi) L_h \sqrt{\frac{K}{n}} \:.
\end{align*}
Combining the inequalities above:
\begin{align*}
    \Pr(h(\xi; f_n) > 0) \leq O(1) B_\alpha^2 (1 + B_\theta L_\phi)^2 L_h^2 \frac{K \log^3{n}}{\gamma^2 n} + O(1) \frac{\log(\log(1+B_h/\gamma)/\delta)}{n} + 1/n^2 \:.
\end{align*}

\section{Proofs for Section~\ref{sec:global}}
\label{sec:app:global}

Before we prove the main results in Section~\ref{sec:global},
we state and prove a few technical lemmas which we will need.
We will let $\mathbb{B}_p^d(x, r)$ denote the $\ell_p$ ball in $\R^d$
centered around $x$ with radius $r$,
and $\leb(\cdot)$ denote the Lebesgue measure on $\R^d$
(with ambient dimension implicit from context).
Let $X \subseteq \R^d$ be full-dimensional and compact.
We denote the \emph{uniform measure} $\mu$ on $X$ to be the measure
defined by $\mu(A) = \frac{\leb(A)}{\leb(X)}$
for every measurable $A \subseteq X$.

\begin{lem}
\label{lem:net}
Fix a $p \in [1, \infty]$.
Let $X \subseteq \R^d$ be a full-dimensional compact set
and let $\mu$ denote its uniform measure.
Let $r_p(\varepsilon)$
be defined as:
\begin{align*}
    r_p(\varepsilon) := \sup_{ U \subseteq X : \mu(U) \leq \varepsilon } \sup\{ r > 0 : \exists x \in U : \mathbb{B}_p^d(x, r) \subseteq U \} \:.
\end{align*}
Then we have that
\begin{align*}
    r_p(\varepsilon) \leq \left( \frac{\varepsilon \leb(X)}{\leb(\mathbb{B}_p^d(0, 1))}  \right)^{1/d} \:.
\end{align*}
\end{lem}
\begin{proof}
Notice that if $r > 0$ satisfies
$\exists x \in U$ such that $\mathbb{B}_p^d(x, r) \subseteq U$,
this implies that
$\leb(U) \geq \leb(\mathbb{B}_p^d(x, r)) = r^d \leb(\mathbb{B}_p^d(0, 1))$.
Hence:
\begin{align*}
    \sup\{ r > 0 : \exists x \in U : \mathbb{B}_p^d(x, r) \subseteq U \} &\leq \sup\{ r > 0 : r^d \leb(\mathbb{B}_p^d(0, 1)) \leq \leb(U) \} \\
    &=  \left( \frac{\leb(U)}{\leb(\mathbb{B}_p^d(0, 1))}  \right)^{1/d} \:.
\end{align*}
Now if $\mu(U) \leq \varepsilon$,
then $\mu(U) = \frac{\leb(U)}{\leb(X)} \leq \varepsilon$
and hence $\leb(U) \leq \varepsilon \leb(X)$.

\end{proof}

\begin{lem}
\label{lem:sphere_net}
Fix a $p \in [1, \infty]$.
Let $X \subseteq \R^d$ be a full-dimensional compact set with $d \geq 2$ and let $\mu$
denote its uniform measure.
Let $\nu := \mu \otimes \varrho$ denote the product measure
on $X \times \mathbb{S}^{d-1}$ with $\varrho$ the Haar measure.
Endow $\R^d \times \mathbb{S}^{d-1}$
with the metric $d(x, y) := \max\left\{ \norm{x_1 - y_1}_p, \rho(x_2, y_2)\right\}$
where $\rho$ is the geodesic distance on $\mathbb{S}^{d-1}$,
and let $B(x, r)$ denote a closed ball of radius $r$ centered at $x$ in this metric space.
Then, the quantity
\begin{equation*}
    r_p(\varepsilon) := \sup_{U \subseteq X \times \mathbb{S}^{d-1} : \nu(U) \leq \varepsilon} \sup \left\{ r > 0 : \exists x \in U : B(x, r) \subseteq U\right\}
\end{equation*}
may be upper bounded by the expression
\begin{equation*}
    r_p(\varepsilon) \leq \sup\left\{ r > 0 : r^d \zeta_d(r) \leq  \frac{\varepsilon\leb(X)}{\leb(\mathbb{B}_p^d(0, 1))} \right\} \:,
\end{equation*}
with the function $\zeta_d : \R_+ \mapsto \R_+$
defined as:
\begin{align*}
   \zeta_d(r) := \begin{cases}
   I\left(\sin^2(r); \frac{d-1}{2}, \frac{1}{2}\right) &\text{if } r \in [0, \pi/2) \\
   1 -  I\left(\sin^2(\pi - r); \frac{d-1}{2}, \frac{1}{2}\right) &\text{if } r \in [\pi/2, \pi) \\
   1 &\text{if } r \geq 1,
   \end{cases} \:,
\end{align*}
where $I(x; a, b)$ is the regularized incomplete beta function.
\end{lem}
\begin{proof}
Let $\mathbb{B}_{\mathbb{S}^{d-1}}(x, r)$ denote
the closed ball in $(\mathbb{S}^{d-1}, \rho)$ centered at $x \in \mathbb{S}^{d-1}$.
Note that $B(x, r) = \mathbb{B}_p^d(x_1, r) \times \mathbb{B}_{\mathbb{S}^{d-1}}(x_2, r)$.
Therefore if $r > 0$ satisfies $\exists x\in U$ such that $B(x, r) \subseteq U$, then
\begin{equation*}
    r^d\frac{\leb(\mathbb{B}_p^d(0, 1))}{\leb(X)}\varrho(\mathbb{B}_{\mathbb{S}^{d-1}}(0, r)) = \frac{\leb(\mathbb{B}_p^d(x_1, r))}{\leb(X)}\varrho(\mathbb{B}_{\mathbb{S}^{d-1}}(x_2, r)) = \nu(B(x, r)) \leq \nu(U).
\end{equation*}
Above, we have used monotonicity of measure and translation invariance of the measure of the respective balls. Now, note that a geodesic ball on $\mathbb{S}^{d-1}$ is a spherical cap, and hence by~\citet{li11cap},
\begin{equation*}
   \varrho(\mathbb{B}_{\mathbb{S}^{d-1}}(0, r)) = \begin{cases}
   I\left(\sin^2(r); \frac{d-1}{2}, \frac{1}{2}\right) &\text{if } r \in [0, \pi/2) \\
   1 -  I\left(\sin^2(\pi - r); \frac{d-1}{2}, \frac{1}{2}\right) &\text{if } r \in [\pi/2, \pi) \\
   1 &\text{if } r \geq 1
   \end{cases} \:.
\end{equation*}
Therefore,
\begin{align*}
    r_p(\varepsilon) \leq \sup\left\{ r > 0 : r^d \zeta_d(r) \leq  \frac{\varepsilon\leb(X)}{\leb(\mathbb{B}_p^d(0, 1))} \right\} \:.
\end{align*}
which proves the result.
\end{proof}

\begin{lem}
\label{lem:ltv_ball}
Let $\dot{x} = A(t) x$ be a linear time-varying system evolving in $\R^n$.
Let $\varphi_t(x)$ denote the flow of this system with
$x(0) = x$. Fix a $t \geq 0$ and a unit vector
$z$. There exists a positive scalar $\alpha$ such that
$\alpha z \in \varphi_t(\mathbb{S}^{n-1})$.
\end{lem}
\begin{proof}
The solution to an LTV system is given by
$x(t) = \Phi(t) x(0)$, where
$\Phi(t) = \exp(\int_0^t A(\tau) \: d\tau)$ and
$\Phi(t)$ is invertible for all $t$.
This means there exists a non-zero $\xi$ such that
$z = \Phi(t) \xi = \Phi(t) \frac{\xi}{\norm{\xi}} \norm{\xi}$.
The claim now follows by taking $\alpha = 1/\norm{\xi}$.
\end{proof}

\subsection{Proof of Theorem~\ref{thm:global_lyap}}

Define $X_g := X \setminus X_b$.
Fix an $x \in \tilde{S}$.
By definition of $\tilde{S}$, there exists a $\xi \in \tilde{X}$ and $t \in T$
such that $\varphi_t(\xi) = x$.
Furthermore, we claim there exists a $\delta_0 > 0$ such
that for all $\delta \in (0, \delta_0)$,
there exists $\xi' \in X_g$ such that
$\norm{\xi - \xi'} \leq r(\varepsilon) + \delta$.
To see this, suppose that
$\xi \in X_b$ (otherwise there is nothing to prove).
By Lemma~\ref{lem:net}, the largest $\ell_2$ ball centered
at $\xi$ contained within $X_b$ has radius at most $r(\varepsilon)$.
Furthermore, by definition of $\tilde{X}$,
$\mathbb{B}_2^p(\xi, r(\varepsilon))$ is
strictly contained within $X$.
This means there exists a $\delta_0 > 0$ such
that for all $\delta \in (0, \delta_0)$,
$\mathbb{B}_2^p(\xi, r(\varepsilon) + \delta) \subset X$.
Furthermore, since $r(\varepsilon)$ is maximal,
then $\mathbb{B}_2^p(\xi, r(\varepsilon) + \delta) \not\subseteq X_b$. This means
that $\mathbb{B}_2^p(\xi, r(\varepsilon) + \delta) \cap X_g$
is non-empty.

We therefore have the following chain of inequalities:
\begin{align*}
    q(\varphi_t(\xi)) &\leq q(\varphi_t(\xi')) + B_{\nabla q} \norm{\varphi_t(\xi) - \varphi_t(\xi')} \\
    &\leq -\lambda V(\varphi_t(\xi')) +  B_{\nabla q} \norm{\varphi_t(\xi) - \varphi_t(\xi')} \\
    &\leq -\lambda V(\varphi_t(\xi)) + (B_{\nabla q} + \lambda B_{\nabla V}) \norm{\varphi_t(\xi) - \varphi_t(\xi')} \\
    &\leq -\lambda V(\varphi_t(\xi)) + (B_{\nabla q} + \lambda B_{\nabla V}) \beta(\norm{\xi - \xi'}, t) \\
    &\leq -\lambda V(\varphi_t(\xi)) + (B_{\nabla q} + \lambda B_{\nabla V}) \beta(r(\varepsilon) + \delta, t) \:.
\end{align*}
Taking the limit as $\delta \to 0$ and using continuity of
$\beta$ with respect to its first argument,
this shows that for any $x \in \tilde{S}$:
\begin{align*}
    q(x) \leq -\lambda V(x) + (B_{\nabla q} + \lambda B_{\nabla V}) \beta(r(\varepsilon), t) \:.
\end{align*}
The claim established by \eqref{eq:V_upper_bound}
now follows from the comparison lemma.
To establish \eqref{eq:lie_derivative_ball},
for any $x \in \tilde{S} \setminus \mathbb{B}_2^p(0, r_b)$,
\begin{align*}
    q(x) &\leq - \lambda V(x) +  (B_{\nabla q} + \lambda B_{\nabla V}) \beta(r(\varepsilon), t) \\
    &\leq- \lambda V(x) +  (B_{\nabla q} + \lambda B_{\nabla V}) \beta(r(\varepsilon), 0 ) \\
    &= - ((1-\eta) \lambda + \eta \lambda) V(x) + (B_{\nabla q} + \lambda B_{\nabla V}) \beta(r(\varepsilon), 0 ) \\
    &= - (1-\eta)\lambda V(x) - \eta \lambda V(x) + (B_{\nabla q} + \lambda B_{\nabla V}) \beta(r(\varepsilon), 0 ) \\
    &\leq - (1-\eta)\lambda V(x) - \eta \lambda \mu r_b^2 + (B_{\nabla q} + \lambda B_{\nabla V}) \beta(r(\varepsilon), 0 ) \:.
\end{align*}
The last inequality follows since $V(x) \geq \mu \norm{x}^2 \geq \mu r_b^2$ for any $x \in \tilde{S} \setminus \mathbb{B}_2^p(0, r_b)$.
The claim \eqref{eq:lie_derivative_ball} now follows
by setting $r_b$ such that
$ - \eta \lambda \mu r_b^2 + (B_{\nabla q} + \lambda B_{\nabla V}) \beta(r(\varepsilon), 0 ) \leq 0$.

\subsection{Proof of Theorem~\ref{thm:global_diff_lyap}}
We begin with a simple lemma, which shows that if a system evolving on Euclidean space is contracting in the metric $M(x, t)$, then the corresponding prolongated system on the tangent bundle will be contracting in a block-diagonal metric.
\begin{lem}
\label{lem:meta}
Let $\dot{x} = f(x, t)$ be a contracting system with rate $\gamma$ on $\calX\subseteq\R^d$ in the metric $M(x, t)$. Then the differential dynamics $\delta\dot{x} = \frac{\p f}{\p x}(x, t)\delta x$ is also contracting in the metric $M(x, t)$ on $\R^d$. Moreover, the prolongated dynamics defined on the tangent bundle
\begin{align*}
    \dot{x} &= f(x, t)\\
    \delta\dot{x} &= \frac{\p f}{\p x}(x, t)\delta x
\end{align*}
is contracting on any compact subset of the tangent bundle $\calX\times\delta \calX\subset \calT\calX \simeq \R^{2p}$.
\end{lem}
\begin{proof}
Consider the differential dynamics $\delta\dot{x} = \frac{\p f}{\p x}(x,t)\delta x$. This system has Jacobian
\begin{equation*}
    \frac{\p \delta \dot{x}}{\p \delta x} = \frac{\p f}{\p x}(x, t),
\end{equation*}
where we have noted that $\delta x \in \calT_{x(t)}\calX\simeq \R^p$ is independent of $x$. This Jacobian induces the second-order variational dynamics
\begin{equation*}
    \delta\delta\dot{x} = \frac{\p f}{\p x}(x, t)\delta\delta x.
\end{equation*}
Consideration of the differential Lyapunov function
\begin{equation*}
    V = \delta\delta x^\T M(x, t)\delta\delta x
\end{equation*}
shows that $V$ decreases exponentially by contraction of $f(x, t)$ in the metric $M(x, t)$ and hence that the virtual dynamics are contracting. Let $\Theta(x, t)$ be such that $\Theta^\T\Theta = M$. The metric transformation
\begin{equation*}
    \Theta'(x, t) = \begin{pmatrix} \Theta(x, t) & 0 \\ 0 & \epsilon \Theta(x, t) \end{pmatrix}
\end{equation*}
for $\epsilon > 0$ leads to the generalized Jacobian
\begin{equation*}
    J'(x, \delta x, t) = \begin{pmatrix} \Theta \frac{\p f}{\p x} \Theta^{-1} + \dot{\Theta}\Theta^{-1} & 0 \\ \epsilon \Theta\frac{\p^2 f}{\p x^2}\delta x \Theta^{-1} & \Theta \frac{\p f}{\p x} \Theta^{-1} + \dot{\Theta}\Theta^{-1} \end{pmatrix},
\end{equation*}
where $\left(\frac{\p^2 f}{\p x^2}\delta x\right)_{ij} = \sum_k \frac{\p^2 f_i}{\p x_j \p x_k}\delta x_k$. Let $Q(x, t) = \Theta \frac{\p f}{\p x} \Theta^{-1} + \dot{\Theta}\Theta$. Contraction of $f$ in the metric $M$ ensures that  $Q(x, t) \leq - \gamma I$, and hence for any vector $(y^\T, z^\T)^\T \in \R^{2p}$,
\begin{align*}
    (y^\T, z^\T)^\T J'(x, \delta x, t)\begin{pmatrix} y\\z\end{pmatrix} &= x^\T Q(x, t) x + y^\T Q(x, t) y + \epsilon y^\T \left[\Theta\frac{\p^2 f}{\p x^2}\delta x \Theta^{-1} \right]x,\\
    &\leq -\gamma \left(1 - \frac{\epsilon \Vert\Theta\Vert\Vert\Theta^{-1}\Vert\Vert\frac{\p^2 f}{\p x^2}\Vert\Vert\delta x\Vert}{2}\right)\left( \Vert x\Vert^2 + \Vert y\Vert^2 \right),
\end{align*}
which shows that the prolongated system is contracting over any compact domain for $\epsilon$ sufficiently small. In particular, for contraction with rate $\eta\gamma$ for $0 < \eta < 1$, we may set
\begin{equation*}
    \epsilon = \frac{2\left(1-\eta\right)}{\Vert\Theta\Vert\Vert\Theta^{-1}\Vert\Vert\frac{\p^2 f}{\p x^2}\Vert
    \sup_{\delta x \in \delta\calX}\Vert\delta x\Vert}.
\end{equation*}
Furthermore, note that the metric transformation $\Theta'$ corresponds to the block-diagonal metric
\begin{equation*}
    M'(x, t) = \begin{pmatrix} M(x, t) & 0 \\ 0 & \epsilon^2 M(x, t)\end{pmatrix}
\end{equation*}
\end{proof}
The proof of Lemma~\ref{lem:meta} imposes a metric $M'(x, t)$ on the second tangent bundle. This construction exploits that the tangent bundle $\TM \simeq \calX \times \R^d$ given that $\calX \subseteq \R^d$, and hence that the tangent bundle can be described by a single global chart. It is immediate to check that this block-diagonal metric is not invariant under a differentiable change of coordinates between overlapping local parametrizations of a general manifold, and hence the proof does not apply beyond Euclidean space. Canonical metrics on $\TM$ such as the Sasaki metric or the Cheeger-Gromoll metric may provide a natural generalization of this proof technique to arbitrary differentiable manifolds~\citep{kappos}.

We now turn to the proof of Theorem~\ref{thm:global_diff_lyap}.
Define $Z_g := (X \times \mathbb{S}^{p-1}) \setminus Z_b$.
Let $t \in T$ and $\xi \in X$ be such that $\xi \in \tilde{X}_t(r_b)$.
Let $\delta \xi \in \mathbb{S}^{p-1}$ be arbitrary.
By Lemma~\ref{lem:sphere_net}
and a similar argument made in the proof in Theorem~\ref{thm:global_lyap},
there exists
a $\delta_0 > 0$ such that
for all $\delta \in (0, \delta_0)$,
there exists
a $(\xi', \delta \xi') \in Z_g$
such that
$\bignorm{ \cvectwo{ \xi - \xi' }{\delta \xi - \delta \xi'} } \leq \sqrt{2} r(\varepsilon) + \delta$ .
By Lemma~\ref{lem:meta}, the prolongated system on the tangent bundle is
exponentially contracting
in a metric $M'$,
so that there exists an $\alpha > 0$ such that
the prolongated system satisfies Assumption~\ref{assume:incr_stability}
with $\beta(s, t) = \sqrt{\chi(M')} s e^{-\alpha t}$.
Here, $\chi(M') = \frac{\sup_x \lambda_{\max} \left\{M'(x)\right\}}{\inf_x \lambda_{\min} \left\{M'(x)\right\}}$ is the condition number of $M'$.
We will derive bounds on $\chi(M')$ and $\alpha$ later in the proof.
Recall that the metric $M(x) \succeq \mu I$,
and therefore $V(x, \delta x) \geq \mu \norm{\delta x}^2$.
Then, by the same argument as in the proof of Theorem~\ref{thm:global_lyap}, for any fixed $\eta \in (0, 1)$:
\begin{align*}
    q(\psi_t(\xi, \delta \xi)) &\leq - \lambda V(\psi_t(\xi, \delta \xi)) + (B_{\nabla q} + \lambda B_{\nabla V}) \beta(\sqrt{2} r(\varepsilon) + \delta, 0), \\
    &\leq -\lambda(1-\eta) V(\psi_t(\xi, \delta \xi)) - \eta \lambda \mu \norm{\theta_t(\delta \xi; \xi)}^2 + (B_{\nabla q} + \lambda B_{\nabla V}) \beta(\sqrt{2} r(\varepsilon) + \delta, 0) \\
    &\leq  -\lambda(1-\eta) V(\psi_t(\xi, \delta \xi)) - \eta \lambda \mu r_b^2 + (B_{\nabla q} + \lambda B_{\nabla V}) \beta(\sqrt{2} r(\varepsilon) + \delta, 0) \:,
\end{align*}
where the last inequality follows since $\xi \in \tilde{X}_t(r_b)$.
Taking the limit as $\delta \to 0$,
\begin{align*}
    q(\psi_t(\xi, \delta \xi)) &\leq -\lambda(1-\eta) V(\psi_t(\xi, \delta \xi)) - \eta \lambda \mu r_b^2 + (B_{\nabla q} + \lambda B_{\nabla V}) \beta(\sqrt{2} r(\varepsilon), 0) \:.
\end{align*}
To find an expression for the condition number and for the contraction rate $\alpha$, consider the block diagonal metric from Lemma~\ref{lem:meta}, $M'(x, t) = \begin{pmatrix} M_\star(x, t) & 0 \\ 0 & \epsilon^2 M_\star(x, t)\end{pmatrix}$.
Following Lemma~\ref{lem:meta}, with $\epsilon = \frac{2(1-\zeta)}{\Vert\Theta(x, t)\Vert\Vert\Theta(x, t)^{-1}\Vert\Vert\frac{\p^2 f}{\p x^2}\Vert\sup_{\delta x\in \delta X}\Vert\delta x\Vert}$ where $\Theta(x, t)^\T\Theta(x, t) = M_\star(x, t)$, the prolongated system will be contracting with rate $\alpha = \zeta\gamma$ for $0 < \zeta < 1$.
Since $m I\preceq M_\star(x, t) \preceq L I$ so that $\chi(M_\star) = \frac{L}{m}$, we immediately have $\epsilon^2 m I \preceq M'(x, t) \preceq LI$.
The condition number $\chi(M')$ then simplifies to $\frac{L}{m\epsilon^2} = \frac{L}{m}\left(\Vert\Theta(x, t)\Vert\Vert\Theta(x, t)^{-1}\Vert\Vert\frac{\p^2 f}{\p x^2}\Vert\sup_{\delta x\in\delta X}\Vert\delta x\Vert\right)^2\frac{1}{4(1-\zeta)^2} = \left(\frac{L}{m}\right)^2B^2_H\sup_{\delta x\in\delta X}\Vert\delta x\Vert^2\frac{1}{4(1-\zeta)^2}$, where we have used that $M_\star = \Theta^\T\Theta$ implies $\Vert\Theta\Vert = \sqrt{L}$, $\Vert\Theta^{-1}\Vert = \sqrt{\frac{1}{m}}$.
By contraction of the variational system in the metric $M_\star(x, t)$ (see Lemma~\ref{lem:meta}), and noting that $0$ is an equilibrium point of the variational dynamics, $\norm{\theta_t(\delta x; \xi)} \leq \sqrt{\frac{L}{m}}$ for all $t$ if $\delta x \in \mathbb{S}^{p-1}$.
Hence we may take $\delta X = \mathbb{B}_2^p\left(0, \sqrt{\frac{L}{m}}\right)$, and conclude that $\chi(M') \leq \left(\frac{L}{m}\right)^3B_H^2\frac{1}{4(1-\zeta)^2} = \chi(M_\star)^3B^2_H\frac{1}{4(1-\zeta)^2}$.

Now with this expression in hand, we choose $r_b$ such that
\begin{align*}
     - \eta \lambda \mu r_b^2 + (B_{\nabla q} + \lambda B_{\nabla V}) \sqrt{2}r(\varepsilon)\chi(M_\star)^{3/2}B_H\frac{1}{2\left(1-\zeta\right)} \leq 0 \:.
\end{align*}
From this we conclude for every $t \in T$ and $\xi \in X$
such that $\xi \in \tilde{X}_t(r_b)$, for every $\delta \xi \in \mathbb{S}^{p-1}$,
\begin{align}
    q(\psi_t(\xi, \delta \xi)) \leq - \lambda (1 - \eta) V(\psi_t(\xi, \delta \xi)) \label{eq:q_ineq_contraction} \:.
\end{align}

To finish the proof,
let $(x, \delta x) \in \tilde{S}(r_b) \times \mathbb{S}^{p-1}$ be arbitrary.  Let $t \in T$ and $\xi \in \tilde{X}_t(r_b)$ such that $x = \varphi_t(\xi)$. By Lemma~\ref{lem:ltv_ball}, let $\delta \varrho \in \mathbb{S}^{p-1}$ be such that
there exists an $\alpha \neq 0$
satisfying $\theta_t(\delta \varrho; \xi) = \alpha \delta x$.
Observe then that
$(x, \alpha \delta x) = (\varphi_t(\xi), \theta_t(\delta \varrho; \xi)) = \psi_t(\xi, \delta \varrho)$
and therefore
by \eqref{eq:q_ineq_contraction},
\begin{align*}
    q(x, \alpha \delta x) =  q(\psi_t(\xi, \delta \varrho)) \leq -\lambda (1-\eta) V(\psi_t(\xi, \delta \varrho)) = -\lambda (1-\eta) V( x, \alpha \delta x) \:.
\end{align*}
By $2$-homogeneity of the inequality above, we may divide
by $\alpha^2$ on both sides to conclude:
\begin{align*}
    q(x, \delta x) \leq - \lambda (1 - \eta) V(x, \delta x) \:.
\end{align*}
Since this inequality holds for arbitrary $x \in \tilde{S}(r_b)$ and
$\delta x \in \mathbb{S}^{p-1}$,
\begin{align*}
    \frac{\p f}{\p x}^\T M(x) + M(x) \frac{\p f}{\p x} + \dot{M}(x) \preceq - 2 (1-\eta) \lambda M(x) \:\:\forall x \in \tilde{S}(r_b) \:.
\end{align*}

\section{Known dynamics}
\label{sec:app:known_dynamics}

In Section~\ref{sec:global}, we assume access only to trajectories. Here we prove a simple proposition under the assumption that the dynamics is known, so that the defining metric condition for contraction can be sampled directly.
\begin{prop}
\label{prop:local_global_met_cond}
Let $M(x, t)$ be a uniformly positive definite matrix-valued function satisfying $M(x, t) \succeq l I$.
Suppose that $X \subseteq \R^p$ is full-dimensional, and let $\dot{x} = f(x, t)$ denote a dynamical system evolving on $X$. Let $\varphi_t(\cdot)$ denote the corresponding flow, let $\nu$ denote the uniform measure on $X$, and let
\begin{align*}
    R(\xi, t) &:= \frac{\p f}{\p x}(\varphi_t(\xi), t)^\T M(\varphi_t(\xi), t) + M(\varphi_t(\xi), t)\frac{\p f}{\p x}(\varphi_t(\xi), t) + \dot{M}(\varphi_t(\xi), t) + 2\lambda M(\varphi_t(\xi), t),\\
    X_b &:= \left\{ \xi \in X : \max_{t \in T}\lambda_{\max} \left\{R(\xi, t)\right\} > 0\right\}.
\end{align*}
Suppose that $\nu\left(X_b\right) \leq \varepsilon$ for some $\varepsilon \in [0, 1]$. Let $M$, $\nabla M$, and $\frac{\p f}{\p x}$ be $L_M$, $L_{\nabla M}$, and $L_J$-Lipschitz continuous, respectively. Further assume that $\Vert M\Vert$, $\Vert\frac{\p f}{\p x}\Vert$, and $\Vert \nabla M\Vert$ are $B_M$, $B_J$, and $B_{\nabla M}$ uniformly bounded in $x$ and $t$, respectively.
Define $r(\varepsilon) := \left( \frac{\varepsilon \leb(X)}{\leb(\mathbb{B}_2^p(1))} \right)^{1/p}$
and let $\tilde{X} := \{ \xi \in X : \mathbb{B}_2^p(\xi, r(\varepsilon)) \subset X \}$.
Then for every $x \in \tilde{S} := \cup_{t \in T} \varphi_t(\tilde{X})$,
the system will be contracting in the metric $M(x, t)$ with a rate $\lambda/\alpha$ for any $\alpha > 1$ if
\begin{equation*}
    \varepsilon \leq \left(\frac{2\lambda l(\alpha-1)}{\alpha\left(2\lambda L_M + L_{\nabla M}B_f + B_{\nabla  M}L_f + 2L_{J}B_M + 2L_M B_J\right)}\right)^p\frac{\pi^{p/2}}{\Gamma\left(\frac{p}{2}+1\right)\leb(X)}.
\end{equation*}
\end{prop}
\begin{proof}
Let us partition $X$ into subsets
\begin{align*}
    X_g &:= \left\{ \xi \in X : \max_{t \in T} \lambda_{\max}\left\{R(\xi, t)\right\} \leq 0\right\}, \\
    X_b &:= X \setminus X_g \:.
\end{align*}
i.e., for any trajectory originating in $X_g$, the metric condition $R(\xi, t)$ remains negative definite along the entire trajectory with rate $\lambda$.
Fix an $x \in \tilde{S}$, for which there exists
a $t \in T$ and $\xi_b \in \tilde{X}$ such that $\varphi_t(\xi) = x$.
By Lemma~\ref{lem:net},
there exists a $\delta_0 > 0$ such
that for all $\delta \in (0, \delta_0)$, there exists
a $\xi_g \in X_g$ such that
$\norm{\xi - \xi_g} \leq r(\varepsilon) + \delta$.  Then,
\begin{align*}
    &\phantom{=}\frac{\p f}{\p x}(\varphi_t(\xi_b), t)^\T M(\varphi_t(\xi_b), t) + M(\varphi_t(\xi_b), t)\frac{\p f}{\p x}(\varphi_t(\xi_b), t) + \dot{M}(\varphi_t(\xi_b), t)\\
    &=\frac{\p f}{\p x}(\varphi_t(\xi_g), t)^\T M(\varphi_t(\xi_g), t) + M(\varphi_t(\xi_g), t)\frac{\p f}{\p x}(\varphi_t(\xi_g), t) + \dot{M}(\varphi_t(\xi_g), t)\\
    &\phantom{=} + \frac{\p f}{\p x}(\varphi_t(\xi_b), t)^\T M(\varphi_t(\xi_b), t) + M(\varphi_t(\xi_b), t)\frac{\p f}{\p x}(\varphi_t(\xi_b), t) + \dot{M}(\varphi_t(\xi_b), t)\\
    &\phantom{=} -\left(\frac{\p f}{\p x}(\varphi_t(\xi_g), t)^\T M(\varphi_t(\xi_g), t) + M(\varphi_t(\xi_g), t)\frac{\p f}{\p x}(\varphi_t(\xi_g), t) + \dot{M}(\varphi_t(\xi_g), t)\right)
\end{align*}
We now control the difference of the terms on the second and third lines of the above equality. To simplify notation, denote $M_g = M(\varphi_t(\xi_g), t)$, with analogous shorthands for $\frac{\p f}{\p x}$ and for subscript $b$. Then, we have that
\begin{align*}
    M_b\frac{\p f}{\p x}_b - M_g\frac{\p f}{\p x}_g &= M_b\left(\frac{\p f}{\p x}_b - \frac{\p f}{\p x}_g\right) + \left(M_b - M_g\right)\frac{\p f}{\p x}_g\\
    &\leq \left(L_{J}B_M + L_M B_J\right)\Vert\xi_g-\xi_b\Vert I,
\end{align*}
with an identical bound for the transpose. Now let $\ip{\nabla M}{\dot{x}}$ denote the tensor contraction $\ip{\nabla M}{\dot{ x}}_{ij} = \ip{\nabla M_{ij}}{\dot{x}}$. Then,
\begin{align*}
    \dot{M}_b - \dot{M}_g &= \ip{\nabla M_b}{f_b} - \ip{\nabla M_g}{f_g}\\
    &= \ip{\nabla M_b - \nabla M_g}{f_b} - \ip{\nabla M_g}{f_b - f_g}\\
    &\leq \left(L_{\nabla M}B_f + B_{\nabla M}L_f\right)\Vert\xi_g-\xi_b\Vert I
\end{align*}
Furthermore, $M_g - M_b \leq L_M\Vert\xi_b-\xi_g\Vert$. Putting these bounds together, we find that
\begin{align*}
    &\phantom{=}\frac{\p f}{\p x}(\varphi_t(\xi_b), t)^\T M(\varphi_t(\xi_b), t) + M(\varphi_t(\xi_b), t)\frac{\p f}{\p x}(\varphi_t(\xi_b), t) + \dot{M}(\varphi_t(\xi_b), t)\\
    &\leq\frac{\p f}{\p x}(\varphi_t(\xi_g), t)^\T M(\varphi_t(\xi_g), t) + M(\varphi_t(\xi_g), t)\frac{\p f}{\p x}(\varphi_t(\xi_g), t) + \dot{M}(\varphi_t(\xi_g), t)\\
    &\phantom{=} + \left(L_{\nabla M}B_f + B_{\nabla M}L_f + 2L_{J}B_M + 2L_M B_J\right)\Vert \xi_g - \xi_b\Vert I\\
    &\leq -2\lambda M(\varphi_t(\xi_g), t) + \left(L_{\nabla M}B_f + B_{\nabla M}L_f + 2L_{J}B_M + 2L_M B_J\right)\Vert \xi_g - \xi_b\Vert I\\
    &\leq -2\lambda M(\varphi_t(\xi_b), t) + \left(2\lambda L_M + L_{\nabla M}B_f + B_{\nabla M}L_f + 2L_{J}B_M + 2L_M B_J\right)\Vert \xi_g - \xi_b\Vert I\\
    &\leq -2\left(\frac{\lambda}{\alpha}\right) M(\varphi_t(\xi_b), t) \\
    &\qquad+ \left[\left(2\lambda L_M + L_{\nabla M}B_f + B_{\nabla  M}L_f + 2L_{J}B_M + 2L_M B_J\right)\Vert \xi_g - \xi_b\Vert - 2\left(\frac{\alpha-1}{\alpha}\right)\lambda l \right]I \:.
\end{align*}
Hence, for contraction at all points $x \in \tilde{S}$ with a rate $\frac{\lambda}{\alpha}$, we require that:
\begin{align*}
    0 & \geq \left(2\lambda L_M + L_{\nabla M}B_f + B_{\nabla M}L_f + 2L_{J}B_M + 2L_M B_J\right)r(\varepsilon) - 2\left(\frac{\alpha-1}{\alpha}\right)\lambda l\\
    &\phantom{=}\mathclap{\Updownarrow} \\
     2\left(\frac{\alpha-1}{\alpha}\right)\lambda l &\geq \left( \frac{\varepsilon \leb(X)}{\leb(\mathbb{B}_2^p(0, 1))}\right)^{1/p}\left(2\lambda L_M + L_{\nabla M}B_f + B_{\nabla M}L_f + 2L_{J}B_M + 2L_M B_J\right)\\
         &\phantom{=}\mathclap{\Updownarrow} \\
     \varepsilon &\leq \left(\frac{2\lambda l(\alpha-1)}{\alpha\left(2\lambda L_M + L_{\nabla M}B_f + B_{\nabla  M}L_f + 2L_{J}B_M + 2L_M B_J\right)}\right)^p\frac{\leb\left(\mathbb{B}_2^p(0, 1)\right)}{\leb(X)} \:.
\end{align*}
\end{proof}

\end{document}